\documentclass[lettersize,journal]{IEEEtran}
\usepackage{amsmath,amsfonts}
\usepackage{algorithmic}
\usepackage{algorithm}
\usepackage{array}
\usepackage[caption=false,font=normalsize,labelfont=sf,textfont=sf]{subfig}
\usepackage{textcomp}
\usepackage{stfloats}
\usepackage{url}
\usepackage{cite}
\usepackage{verbatim}
\usepackage{graphicx}
\usepackage{amsthm}
\usepackage{booktabs}
\usepackage{color}
 
\usepackage{algorithmic}
 
\makeatletter
\newcommand{\removelatexerror}{\let\@latex@error\@gobble}
\makeatother
\begin{document}

\title{Distributed Robust Learning-Based Backstepping Control Aided with Neurodynamics for Consensus Formation Tracking of Underwater Vessels}

\author{Tao Yan,~\IEEEmembership{Graduate Student Member,~IEEE}, Zhe Xu,~\IEEEmembership{ Member,~IEEE}, Simon X. Yang,~\IEEEmembership{Senior Member,~IEEE}
\thanks{This work was supported by the Natural Sciences and Engineering Research Council (NSERC) of Canada. \textit{(Corresponding author: Simon X. Yang.)}}
\thanks{T. Yan and S. X. Yang are with the Advanced Robotics and Intelligent Systems (ARIS) Laboratory, School of Engineering, University of Guelph, Guelph, ON N1G2W1, Canada (e-mails: tyan03@uoguelph.ca; syang@uoguelph.ca).}
\thanks{Z. Xu is with the Intelligent and Cognitive Engineering (ICE) Laboratory, Department of Mechanical Engineering, McMaster University, Hamilton, ON L8S4L8, Canada (e-mail: xu804@mcmaster.ca).}}



\maketitle

\begin{abstract}
This paper addresses distributed robust learning-based control for consensus formation tracking of multiple underwater vessels, in which the system parameters of the marine vessels are assumed to be entirely unknown and subject to the modeling mismatch, oceanic disturbances, and noises. Towards this end, graph theory is used to allow us to synthesize the distributed controller with a stability guarantee. Due to the fact that the parameter uncertainties only arise in the vessels' dynamic model, the backstepping control technique is then employed. Subsequently, to overcome the difficulties in handling time-varying and unknown systems, an online learning procedure is developed in the proposed distributed formation control protocol. Moreover, modeling errors, environmental disturbances, and measurement noises are considered and tackled by introducing a neurodynamics model in the controller design to obtain a robust solution. Then, the stability analysis of the overall closed-loop system under the proposed scheme is provided to ensure the robust adaptive performance at the theoretical level. Finally, extensive simulation experiments are conducted to further verify the efficacy of the presented distributed control protocol.
\end{abstract}


\section{Introduction}
\IEEEPARstart{A}{utonomous} underwater vessels (AUVs) are referred to as the unmanned devices capable of performing specific missions automatically offshore or even in the deep sea environments for a long period of time. Due to that of capabilities, such systems have been applied to many practical productions and processes over the past few decades, such as oceanographic mapping, oil and gas exploration, submarine pipeline inspection, and even for military purposes \cite{8809889,4,8372958}. Nevertheless, as the increase of task complicity as well as the demand for high reliable sensing capabilities, more expensive or \textit{ad hoc} ships are sometimes required to guarantee a quality completion of assigned tasks. Recently, as an efficient alternative to the employment of such tailored devices, multiple relatively simple, small and cheap AUVs are used to construct a fleet to accomplish the corresponding missions in a collaborative way \cite{5,zhu2020novel}. Apart from the aforementioned features, such systems also are of several inherent properties, including ease of scalability, robust data collection, wide-area coverage, good fault-tolerant ability, etc. The major challenges of applying multiple vessel systems rely on the fact that it is imperative to synthesize efficient coordination strategies as well as motion control algorithms such that the individuals in the fleet can be driven to work together for common objectives. Formation tracking control, identified as one of the fundamental problems behind multi-AUVs coordination and cooperation, has attracted considerable attention in recent decades \cite{6,9709103}. While such a multi-agent coordination problem can also be found in other robotic platforms, e.g., unmanned ground robots, unmanned aerial vehicles, and spacecraft, due to more complicated and unpredictable underwater conditions as well as the nonlinear uncertain characteristic of AUVs, the development of high performance formation tracking control protocols for such systems may be more challenging and is still open for the societies of control and ocean engineering \cite{lakhekar2019disturbance,wang2020active}.

Roughly speaking, the formation control of an underwater marine vessel fleet can be typically divided into two portions, that is, coordination strategies and motion control schemes. As to the former, there are a few commonly used methodologies for coordinating multiple vessels to form a certain configuration, such as leader-following method \cite{9694518,9557752,9632455,8651430}, virtual structure method \cite{9,10}, behavior-based approaches \cite{6,8}, artificial potential field approaches \cite{13,14}, etc. In addition to the group coordination, owing to the highly nonlinear hydrodynamic characteristic of the AUVs as well as the unpredictable marine conditions, there is also a pressing need for efficient and robust motion control schemes to drive the vessels to reach and maintain the prescribed formation precisely. To tackle these technical challenges in control, Millán \textit{et al.} proposed a virtual leader based H$_2$/H$_{\infty}$ optimal control scheme with a feedforward compensator to steer fleets of AUVs to form a formation so that the communication issues, i.e., package dropouts and delays can be addressed \cite{6547175}. While the linear quadratic based optimal solution can yield an effective formation performance, it is merely suitable for restrictive operating conditions, that is, only limited local stability properties can be guaranteed. To extend to a broader operating area, nonlinear control techniques have received much attention in the last few decades. The formation tracking problem of multiple underwater vessels was addressed in \cite{5246389} where the goal of vessels is not only to maintain a desired spatial formation pattern but also to track a set of waypoints using a line-of-sight strategy, for which the leader-following modeling method is utilized and on the top of that, a feedback linearization based nonlinear controller was then derived to ensure the globally asymptotic stability. Formation tracking control was studied and a Lyapunov-based model predictive controller was developed, where an extended state observer was incorporated so that the proposed controller not only obtained an optimal performance but also with certain robustness against the maritime disturbances \cite{8878014}. In this work, the authors assumed that the AUVs modeling information is able to be accessed.  To address the variable added mass and poor communication capacities, an adaptive sliding mode control (SMC) protocol was developed by means of the superb robustness properties of SMC techniques to any bounded matched disturbances \cite{22}.

However, almost all of the aforementioned methods employ either a simplified dynamic model or a kinematics-based model to design the corresponding formation controllers, which unavoidably leads to a more restrictive control design and makes it unlikely to track a fast varied 3-dimensional (3D) trajectory. Indeed, designing formation controllers for marine vessels in 3D space with full dynamic models is more challenging due to more degrees of freedom (DOF) and uncertainties to be tackled. Towards this end, Hou and Cheah developed an adaptive proportional-derivative control scheme for multi-AUVs formation control on the basis of a completed dynamic model with 6-DOF where less knowledge regarding the plant is used, that is, with some uncertainties in gravitational, buoyancy forces and oceanic disturbances \cite{5}. While the method presented has fewer control parameters whose physical meanings are also clear, the resulting formation accuracy is not always good enough owing to the time-varying uncertainties. To attain a more accurate performance, an adaptive neural network-based solution was provided \cite{16}, in which the neural network was incorporated in the formation control design to approximate the part of nonlinear uncertainties resulting from the frictions, marine disturbances, and unmodeled dynamics. Nonetheless, the derived formation protocol was based on a virtual leader scheme for which each vehicle in the fleet was treated as an independent individual and there are no actual connections between neighbors. On the other hand, considering the unavailability of velocity measurements in practice, an extended state observer (ESO) based integral sliding mode control (ISMC) method was proposed, in which the ESO was aimed to provide real-time estimations for both vessels' velocities and the external disturbances, followed then by an ISMC to adaptively handle the rest of the internal uncertainties \cite{cui2017extended}. Despite the fact that the SMC based control schemes possess good robustness, such methods always suffer from the chattering issue, which may excite the unmodeled high frequency dynamics of the systems in practice. Therefore, adaptive higher-order SMC schemes were developed based on a gain adaptation mechanism to mitigate the chattering adequately while maintaining the sliding mode as much as possible \cite{EDWARDS2016183}. While the chattering can be attenuated, the resulting controllers rely on the assumption of a bounded derivative of disturbances and are, besides, still quite sensitive to measurement noises, both of which significantly restrict their applications to many practical situations.

To the best of our knowledge, robust learning enabled consensus formation tracking control of AUVs fleet in 3D space has not been sufficiently resolved in the literature. As discussed above, the results obtained are not applicable to the situation studied in this paper. The main contributions are summarized as follows:
\begin{enumerate}
\item{ A novel distributed robust learning based control methodology is first proposed to address the formation control problem considered, in which it is assumed that the system parameters of AUVs are completely unknown and subject to the modeling errors, environmental disturbances, and measurement noises.}
\item{An online learning procedure is developed in the control loop, responsible for the real-time estimation of plant parameters so that a better steady-state formation accuracy can be expected.}
\item{Backstepping technique is employed to facilitate the learning based nonlinear control design. Moreover, the rest of system uncertainties, including modeling errors, external disturbances, and noises, are addressed effectively by a neurodynamics based robust controller.}
\item{Rigorous stability analysis for the resulting closed-loop formation system is conducted using the Lyapunov stability theory to guarantee robust formation performance at the theoretical level.}
\end{enumerate}

The rest of the article is outlined as follows. Some basic knowledge of graph theory is presented and the formation control problem considered is formulated in Section \ref{s2}. In Section \ref{s3}, an online learning procedure is developed for each AUV. Section \ref{s4} addresses the learning based formation tracking of fleets of underwater vessels subject to both modeling mismatch and exogenous disturbances. Section \ref{s5} provides extensive simulation validations, and Section \ref{s6} concludes this work.

\section{PRELIMINARY AND PROBLEM FORMULATION}\label{s2}
In this section, the basic knowledge regarding the graph theory is presented briefly. The mathematical model of AUVs used for formation control design is described, and moreover the objective of considered formation tracking control of AUV fleet is formulated.

\subsection{Preliminary on graph theory}\label{s2.1}
The communication topology established among the individuals in a fleet of marine vessels can be modeled by a weighted directed graph $G = \{ {V,E,A} \} $, thus constructing a networked system, and each vessel in such a system can be treated as a node. As for a simple time-invariant graph $G$, it is characterized by the vertex set $V = \{ { \nu _1, \nu _2, \ldots, \nu _N } \}$ , the edge set $E \subseteq V \times V$, and the weighted adjacency matrix $A = \left[ {{a_{ij}}} \right] \in \mathbb{R}^{\text{N} \times \text{N}}$. The element $\nu _i$ in vertex set $V$ denotes $i$-th AUV, and the index $i$ belongs to an accountable index set $\Gamma  = \left\{ {1, \ldots ,N} \right\}$. If there exists the information exchange between AUV $i$ and AUV $j$, then, say, there is an edge between AUVs $i$ and $j$, i.e., $\left( {{\nu _i},{\nu _j}} \right) \in E$, and ${a_{ij}} = {a_{ji}} > 0$. Particularly, call vehicle $j$ a neighbor of vehicle $i$, and the set of neighbors is denoted by ${N_i} = \left\{ {j | {\left( {{\nu _i},{\nu _j}} \right) \in E} } \right\}$. Otherwise, there is no edge among them, and $a_{ij} = a_{ji} = 0$. Moreover, we define $a_{ii} = 0$ for all $i \in \Gamma$, and the out-degree $d_{i} = \sum _{j \in N_i} {a_{ij}}$ associated with node $i$. Afterwards, the degree matrix as well as the Laplacian matrix of the graph $G$ can then be defined as $D = \text{diag} \left\{ {d_1, \ldots, d_N}  \right\} \in \mathbb{R} ^{\text{N} \times \text{N}}$ and $L = D-A$, respectively. A path in graph is a sequence consisted of a set of successive adjacent nodes, starting from node $i$ and ending at node $j$. If any two nodes in a graph $G$ have at least one path, then, say, graph $G$ is connected.

In order to make the AUVs fleet move along with a desired path as a whole, a reference trajectory must be defined ahead of time. The availability to the information of reference trajectory for $i$-th AUV is indicated by a parameter $b_i$; that is, if AUV $i$ is permitted to access this information, then $b_i >0$; otherwise, $b_i = 0$. Define $B = \text{diag} \left\{ {b_1, \ldots, b_N} \right\}$. 
\newtheorem{assumption}{Assumption}
\begin{assumption} \label{assumption1}
For the considered multi-AUV formation control network, graph $G$ is connected, and moreover there is at least one AUV able to receive the information of reference trajectory, i.e., the elements of $B$ are not all equal to zero.
\end{assumption}

\newtheorem{lemma}{Lemma}
\begin{lemma}\label{lemma1}
if Assumption \ref{assumption1} holds, then matrix $L+B$ is positive definite.
\end{lemma}

\begin{figure}[!t]
\centering
\includegraphics[width=2in]{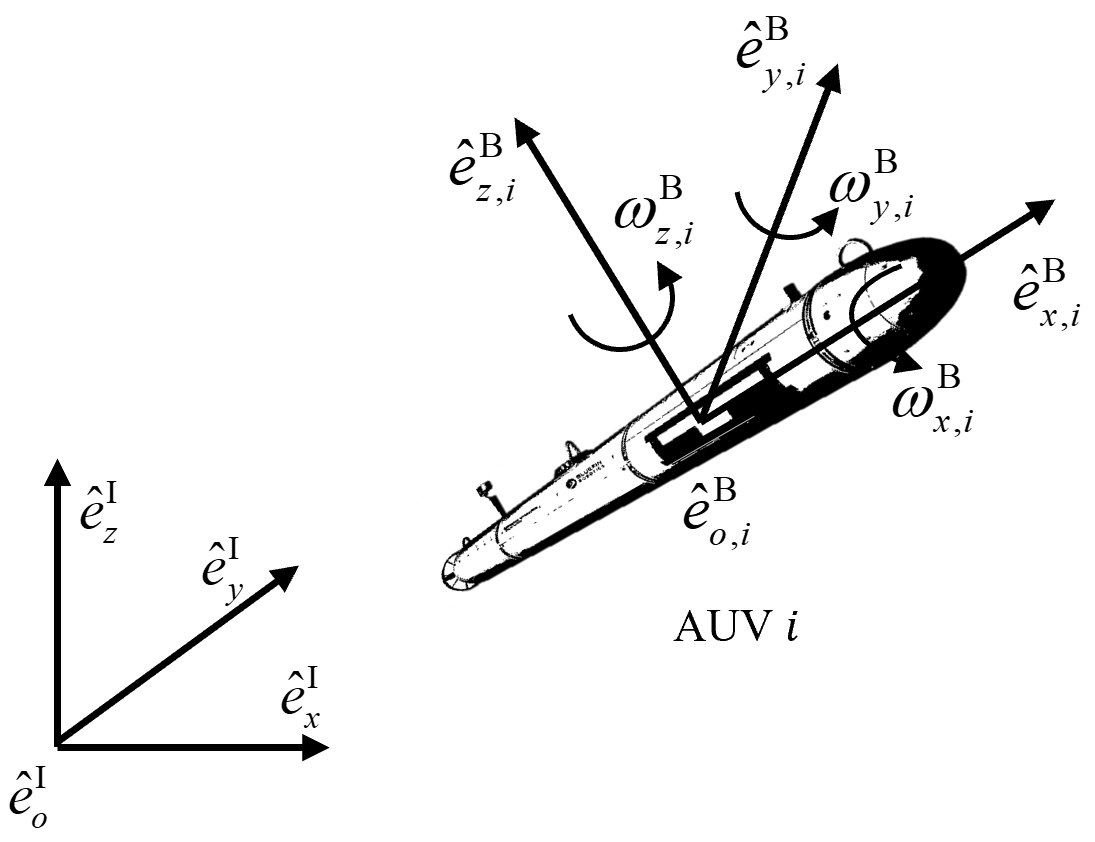}
\caption{Schematic diagram of $i$-th AUV .}
\label{fig1}
\end{figure}

\subsection{Problem formulation}
The robust learning-based consensus formation tracking of $N$ numbers of AUVs in 3-dimensional space is addressed in this article. As presented in the work of Yan \textit{et al.} \cite{BSMC}, the kinematic and dynamic models of $i$-th AUV $(i \in \Gamma)$ can be described as
\begin{align}
& \dot  \eta _i = J_i (\eta_{2,i})v_i,  \label{eq1a} \\
& M_i \dot v_i + C_i (v_i) v_i + D_i (v_i)v_i + G_i(\eta _i)= \tau _i + d_i , \label{eq1b}
\end{align}
where $\eta _i = \left[\eta _{1,i}^{\rm T}, \eta _{2,i}^{\rm T} \right] ^{\rm T} \in \mathbb{R}^{6}$, $\eta _{1,i} = \left[x_i, y_i, z_i \right] ^{\rm T} \in \mathbb{R}^{3}$, $\eta _{2,i} = \left[\phi_i, \theta_i, \psi_i \right] ^{\rm T} \in \mathbb{R}^{3}$ denote the position and orientation of $i$-th AUV, respectively, which are expressed in the Earth-fixed frame $\hat E^{\rm I} = \left\{ {\hat e_o^{\rm I}, \hat e_x^{\rm I}, \hat e_y^{\rm I},\hat e_z^{\rm I}} \right\}$, and $ v_i = \left[v_{1,i}^{\rm T}, v_{2,i}^{\rm T} \right] ^{\rm T} \in \mathbb{R}^{6 }$, $v_{1,i} = \left[v_{x,i}, v_{y,i}, v_{z,i} \right] ^{\rm T} \in \mathbb{R}^{3}$, $v_{2,i} = \left[\omega_{x,i}, \omega_{y,i}, \omega_{z,i} \right] ^{\rm T} \in \mathbb{R}^{3}$ are the $i$-th AUV’s translational and rotational velocities, respectively, described in vessel’s body-fixed frame $\hat E^{\rm B} = \left\{ {\hat e_{o,i}^{\rm B}, \hat e_{x,i}^{\rm B}, \hat e_{y,i}^{\rm B},\hat e_{z,i}^{\rm B}} \right\}$. The sketch of the AUV $i$ is illustrated in Fig. \ref{fig1}. The transformation between two frames is described by the Jacobian matrix $J_i (\eta_{2,i})$. $M_i \in \mathbb{R}^{6 \times 6}$ is the inertia matrix, $C_i(v_i) \in \mathbb{R}^{6 \times 6}$ the Coriolis and centripetal matrix, $D_i(v_i) \in \mathbb{R}^{6 \times 6}$ the hydrodynamic damping matrix, and $G_i(\eta _i) \in \mathbb{R}^{6}$  the gravitational related term. The generalized control input vector is represented by  $\tau _i \in \mathbb{R}^{6}$ , and $d_i \in \mathbb{R}^{6}$ is the lumped disturbance, describing both the modeling errors and exogenous disturbances induced by the wind, waves and ocean currents. The detailed definitions for those matrices can refer to the previous work \cite{BSMC}.

\newtheorem{remark}{Remark}
\begin{remark}\label{remark1}
It is in fact difficult to access the accurate values of above mentioned system matrices, and owing to the hydrodynamic phenomena in practice most of these values may even be subject to variations. To this end, this paper addresses the consensus formation tracking of multiple underwater vehicles where all of these system parameters are assumed to be entirely unknown, not just the hydrodynamic related terms, and besides that, the modeling errors, external disturbances and noises are also taken into consideration to make our approach robust for more practical situations.
\end{remark}

In the problem of formation tracking, a desired formation pattern of a fleet can be determined by a set of predefined relative postures (positions and orientations) between the vessels $i$ and its neighbors $j$, $(i,j \in \Gamma)$; specifically, let the desired postures for pair ($i,j$) be $\Delta _{ij} = \left[ \delta _{x,ij}, \delta _{y,ij}, \delta _{z,ij}, \delta _{\phi,ij}, \delta _{\theta,ij}, \delta _{\psi,ij} \right] ^{\rm T} \in  \mathbb{R}^{6}$. It is noted that the orientation of fleets of vessels should be aligned, that is, the relative attitudes $\left[ \delta _{\phi,ij}, \delta _{\theta,ij}, \delta _{\psi,ij} \right] ^{\rm T} $ of vessels are always set to $\mathbf{0}_{3}$. In addition to the shape maintenance, in many practical missions the fleets are also required to follow a prescribed trajectory. In this respect, let $\eta _{1,i}^d = \left[ x_i^d, y_i^d, z_i^d \right]^{\rm T} \in  \mathbb{R}^{3} $ be the desired second-order-differentiable-bounded trajectory, $\eta _{2,i}^d = \left[ \phi _{i}^d, \theta _{i}^d, \psi _{i}^d \right]^{\rm T} \in \mathbb{R}^{3}$ be the corresponding second-order-differentiable-bounded desired attitude for the $i$-th vessel, and $\eta _i ^d = \left[ \eta _{1,i}^d, \eta _{2,i}^d \right] ^{\rm T}$. The objective of this paper is concerned with synthesizing a distributed control law for $\tau_i$ $\left(i \in \Gamma \right)$ where the parameters of the systems are assumed to be completely unknown and the impacts of modeling errors and environmental disturbances are both considered, such that the following coordinated motion of a fleet of vessels can be achieved
\begin{itemize}
    \item {the preassigned desired relative postures $\Delta _{ij}$ can be formed and maintained,}
    \item {and each vessel is able to follow a predefined trajectory $\eta _i^d$.}
\end{itemize}
We have the following assumption.
\begin{assumption} \label{assump2}
 It is assumed that the lumped disturbance $d_i (t)$ enforced on $i$-th vessel ($i \in \Gamma$) that describes both model mismatching and environmental disturbances is bounded and satisfies
\begin{equation}
    \left\| {{d_i}(t) } \right\|  \le \rho _1,
\end{equation}
where $\rho _1$ is a certain positive constant.
\end{assumption}

\section{PARAMETER ESTIMATOR DESIGN} \label{s3}
This section addresses the online model learning for each individual vessel. To do so, a parameter estimator shall be designed by proposing an effective adaptation scheme so that the parameter unavailable and time-varying issues can be handled in a real-time manner. It is highlighted that all parameters in vessels' dynamic model \eqref{eq1b} are supposed to be unknown and required to be estimated, not just the hydrodynamic related terms. Moreover, the input-to-state stability properties of the proposed parameter estimator is established.

To ease the design of parameter estimator, we may first rewrite the dynamics \eqref{eq1b} into the following linear form with respect to system parameters 
\begin{equation} \label{eq7_}
    \dot v_i = \Psi_i \left( {v_i, \eta _i, \tau_i } \right ) \theta _i ^{\star} + \tilde  d_i ,
\end{equation}
where $\Psi _i \in \mathbb{R}^{6 \times 24}$ is referred to as a regression matrix, depending on the current states and inputs of a vessel. Vector $\theta _i ^{\star} \in \mathbb{R}^{24}$ is the true value of the system parameters, and $\tilde d_i$ describes both the modeling errors and marine disturbances acting on the $i$-th vessel. It is easy to verify that $\tilde{d}_i$ is also bounded in accordance with the Assumption \ref{assump2}, i.e., $\| \tilde{d}_i(t) \|  \le \rho _2$, where $\rho _2$ is some positive constant. $\Psi _i $ is defined as follows
\begin{align}
    \Psi _{i,1} &= \left[ v_{z,i}\omega _{y,i}, v_{y,i} \omega _{z,i}, v_{x,i}, \tau _{1,i}, \text{0}_{20}  \right], \nonumber\\
    \Psi _{i,2} &= \left[ \text{0}_{4}, v_{z,i}\omega _{x,i}, v_{x,i} \omega _{z,i}, v_{y,i}, \tau _{2,i}, \text{0}_{16} \right], \nonumber\\
    \Psi _{i,3} &= \left[ \text{0}_{8}, v_{y,i}\omega _{x,i}, v_{x,i} \omega _{y,i}, v_{z,i}, \tau _{3,i}, \text{0}_{12} \right], \nonumber\\
    \Psi _{i,4} &= \left[ \text{0}_{12}, v_{y,i} v_{z,i}, \omega _{y,i}\omega _{z,i},\omega _{x,i}, \tau _{4,i}, \text{0}_{8} \right], \nonumber\\
    \Psi _{i,5} &= \left[ \text{0}_{16}, v_{x,i} v_{z,i}, \omega _{x,i}\omega _{z,i},\omega _{y,i}, \tau _{5,i}, \text{0}_{4} \right], \nonumber\\
    \Psi _{i,6} &= \left[ \text{0}_{20}, v_{x,i} v_{y,i}, \omega _{x,i}\omega _{y,i},\omega _{z,i}, \tau _{6,i} \right],
\end{align}
where $\Psi _{i,j}$, $j \in \left\{ 1,2,3,4,5,6 \right\} $ represents each row of the regression matrix $\Psi _{i}$, and the subscript of $\text{0}$ in each row indicates the number of consecutive zero elements. 

Due to the assumption that the system parameters are completely unknown for the control synthesis, our first goal is to design an adaptation scheme for parameter vector $\theta _i ^{\star}$ such that the resulting estimate can approach its real value $\theta _i ^{\star}$ as $t \to \infty$. In other words, an online learning process shall be enabled here to provide a real-time estimation for the parameter vector $\theta _i ^{\star}$ based on input-output data, i.e., the pair of $v_i$, $\eta _i$, and $\tau _i$.

For this purpose, we may design a parameter estimator for $i$-th vessel $(i\in \Gamma)$ with the following adaptation law
\begin{align}
    \dot{\hat{v}}_i &= \Psi _i \left( {v_i, \eta _i, \tau _i} \right) \theta _i - L_i \left( {\hat{v}_i - v_i} \right), \label{eq8a_} \\
    \dot \theta _i &= - \Psi _i \left( {v_i, \eta _i, \tau_i } \right )^{\rm T} P_i \left( {\hat{v}_i - v_i} \right), \label{eq8b_}
\end{align}
where $\hat{v}_i \in \mathbb{R}^{6}$ and $\theta _i\in \mathbb{R}^{24}$ are the estimates of $ v_i$ and $\theta _i^{\star}$, respectively, and $L_i \in \mathbb{R}^{6 \times 6}$ and $P_i \in \mathbb{R}^{6 \times 6}$ are the gain matrices of the proposed estimator to be designed.
\begin{remark} \label{remark2_}
It should be stressed that the parameter estimator presented is consisted of two subsystems. Specifically, the first subsystem \eqref{eq8a_} actually is a standard state observer used to observe the state $v_i$, but is derived based on the current estimate $\theta _i$. Since we assume that the full state measurements of $i$-th vessel are available that can be treated as the supervised signals for the parameter estimation, the adaptation scheme for $\theta _i$ is then driven by the deviation between $\hat{v}_i$ and its actual value $v_i$. In other words, the goal now is cast to seek an adaptation law for $\theta _i$ such that the error of $\hat{v}_i$ and $v_i$ could be minimized as $t \to \infty$. In what follows, we show that our proposed adaptive mechanism \eqref{eq8b_} is able to achieve this purpose.
\end{remark}
Define first the $i$-th vessel's observation error ${\tilde{v}_i} = \hat{v}_i - v_i$ and estimation error $\tilde{\theta}_i = \theta _i- \theta _i ^{\star}$, and their derivatives can be readily obtained as $\dot{\tilde{v}}_i = \dot{\hat{v}}_i - \dot{v}_i$ and $\dot{\tilde{\theta}}_i = \dot{\theta}_i$, respectively. Plugging the system dynamics \eqref{eq7_} as well as the parameter estimator \eqref{eq8a_} and \eqref{eq8b_} in, yield the following error dynamics for parameter estimation
\begin{align}
    \dot{\tilde{v}}_i &= \Psi _i\left( {v_i, \eta _i, \tau _i} \right)\tilde{\theta}_i - L_i \tilde{v}_i - \tilde{d}_i, \label{eq9a_} \\
    \dot{\tilde{\theta}}_i &= - \Psi _i\left( {v_i, \eta _i, \tau _i} \right)^{\rm T} P_i \tilde{v}_i. \label{eq9b_}
\end{align}
We then have the following stability properties.
\begin{lemma}\label{lemma2_}
The error dynamics of parameter estimation for $i$-th vessel described by \eqref{eq9a_} and \eqref{eq9b_} is input-to-state stable if gain matrices $L_i$ and $P_i$ are chosen to be positive definite diagonal and Assumption \ref{assump2} holds.
\end{lemma}
\begin{proof}
Propose the Lyapunov function candidate as follows
\begin{equation}\label{eq10_}
    V_{1,i} = \frac{1}{2} \tilde{v}_i^{\rm T} P_i \tilde{v}_i + \frac{1}{2} \tilde{\theta}_i^{\rm T} \tilde{\theta}_i.
\end{equation}
The time derivative of $V_{1,i}$ along the trajectories of error dynamics \eqref{eq9a_} and \eqref{eq9b_} can be obtained
\begin{align}
    \dot{V}_{1,i} &= \tilde{v}_i^{\rm T} P_i \left( {\Psi _i \tilde{\theta}_i - L_i\tilde{v}_i - \tilde{d}_i } \right) - \tilde{\theta}_i^{\rm T} \Psi _i^{\rm T} P_i \tilde{v}_i \nonumber \\
    &= \tilde{v}_i^{\rm T} P_i \Psi _i \tilde{\theta}_i - \tilde{v}_i^{\rm T} P_i  L_i\tilde{v}_i - \tilde{\theta}_i^{\rm T} \Psi _i^{\rm T} P_i  \tilde{v}_i - \tilde{v}_i^{\rm T} P_i \tilde{d}_i \nonumber \\
    &= - \tilde{v}_i^{\rm T} P_i  L_i\tilde{v}_i + \tilde{v}_i^{\rm T} \left( {P_i \Psi _i - P_i^{\rm T} \Psi _i} \right) \tilde{\theta}_i - \tilde{v}_i^{\rm T} P_i \tilde{d}_i \nonumber \\
    &= - \tilde{v}_i^{\rm T} P_i  L_i\tilde{v}_i  - \tilde{v}_i^{\rm T} P_i \tilde{d}_i .\label{eq11_}
\end{align}
Note that for simplicity the arguments of the functions are omitted so long as there is no ambiguity. Let $c_1 = \lambda_{\min}\left({P_i L_i}\right) $ and $c_2 = \lambda_{\max}\left({P_i}\right) $ , where $\lambda_{\min}\left({\cdot}\right)$ and $\lambda_{\max}\left({\cdot}\right)$ denote the minimum and maximum eigenvalues of a matrix, respectively. We then get the following inequality from the bounded disturbance condition
\begin{align}\label{eq12_}
    \dot{ V}_{1,i}  & \le  - c_{1} \left\| {\tilde{v}_i} \right\|^{2} + c_2 \rho_2 \left\|{\tilde{v}_i}\right\|  \nonumber \\
    & \le -\left({c_1 - \alpha}\right) \left\| {\tilde{v}_i} \right\|^{2}, \quad \text{whenever} \quad \left\| {\tilde{v}_i} \right\| \ge \mu .
\end{align}
Here, $\alpha$ is an arbitrary number, satisfying $ 0<\alpha < c_1$, and $\mu = \left({c_2 \rho_2}\right)/ {\alpha}$. Furthermore, we may have
\begin{equation}\label{eq13_}
    \dot{ V}_{1,i}  \le - k, \quad \forall \left\| {\tilde{v}_i} \right\| \ge \mu, \quad \forall \ t \ge t_0,
\end{equation}
where $k=\left({c_1 - \alpha}\right) {\mu}^2$. Then, taking the integral of \eqref{eq13_} yields
\begin{align}
  V_{1,i}\left({t}\right) & \le V_{1,i} \left({t_0}\right) - k\left( {t-t_0} \right), \label{eq14_}
\end{align}
and therefore,
\begin{align}
    \left\| {\tilde{v}_i}(t) \right\| &\le \sqrt {\frac{2V_{1,i} \left({t_0}\right)-2k\left( {t-t_0} \right)}{c_3}}, \label{eq15_}\\
    \left\| {\tilde{\theta}_i}(t) \right\| &\le \sqrt {{2V_{1,i} \left({t_0}\right)-2k\left( {t-t_0} \right)}}, \label{eq16_} 
\end{align}
where $c_3 = \lambda_{\min}\left({P_i}\right)$. It can be readily concluded that both $\tilde{v}_i$ and $\tilde{\theta} _i$ are uniformly ultimately bounded for all $t \ge t_0$. In particular, the ultimate bounds can be further given by
\begin{align}
    \left\| {\tilde{v}_i}(t) \right\| &\le \sqrt {\frac{c_2 \mu ^{2}}{{c_3}} + \frac{1}{c_3}\left\| {\tilde{\theta}_i}(t_0) \right\| ^{2}}, \label{eq17_}\\
    \left\| {\tilde{\theta}_i}(t) \right\| &\le \sqrt {{c_2 \mu ^{2}+ \left\| {\tilde{\theta}_i}(t_0) \right\| ^{2}}}. \label{eq18_} 
\end{align}
This completes the proof.
\end{proof}
\begin{remark}\label{remark3_}
It also follows from \eqref{eq17_} and \eqref{eq18_} that the robustness properties of developed parameter estimator are achieved; that is, under the bounded input $\tilde{d}_i$ both errors of observation and estimation can be maintained within a small neighborhood of the origin by choosing the parameters, i.e., $c_1$,  $c_2$,  $c_3$ and $\mu$, appropriately. Moreover, if we step into the bound on $\tilde{v}_i$ further, it can be seen that the bound is involved with two portions, one of which comes from the effects of disturbances and another from the parameter estimation error. Specifically, the effects from the parameter estimation can be reduced by solely increasing the estimation gain $P_i$, and thus based on the \eqref{eq9a_} this, in turn, implies that the estimation performance can be improved accordingly.
\end{remark}
\begin{remark}
Note that the terminology 'online' or 'real-time' used here lies in the fact that the parameter estimation process is nested into the feedback loop and, besides, merely the current measurement is used to perform the estimation, not relying on the history information of the state trajectories.
\end{remark}

\section{FORMATION CONTROL PROTOCOL DESIGN}\label{s4}
This section addresses the distributed learning-based control for formation tracking of a fleet of AUVs. In such a control problem, there are several pressing difficulties needed to be tackled: 1) The controls occur in a local manner, that is, solely the neighboring information of a vessel is permitted to be accessed for the control synthesis. 2) The dynamic parameters of the marine vessels are assumed to be entirely unknown and even time-varying. 3) It is necessary to consider the impacts of model mismatching, ocean disturbances, and measurement noises in the controller design so as to make the proposed scheme robust to practical scenarios.

To this end, we derive a novel distributed consensus control protocol based on the graph theory as shown in Section \ref{s2.1}. Observe that since the parameter uncertainties only appear in the vessels' dynamic model, the backstepping control technique can be used to help synthesize the controller. Then, a learning procedure as developed in Section \ref{s3} is embedded in the proposed protocol to provide real-time parameter identification. As a result, the issues of parameters unavailable and time-varying can be handled effectively. Furthermore, to improve robustness a neurodynamics-based compensator is introduced. Finally, the input-to-state stability of the resulting overall closed-loop system is proved using the Lyapunov theory.


\subsection{Distributed learning-based control with neurodynamics}
To accomplish the anticipated control objectives, we first define the consensus formation tracking error for $i$-th vessel $(i \in \Gamma)$, as follows, which is aimed to be minimized 
\begin{equation}\label{eq24_}
    e_{i} = \sum\limits_{j \in N_i} {a_{ij} \left(\eta _i- \eta _j -\Delta _{ij} \right) + b_i \left(\eta _i -\eta _i ^d \right)}, 
\end{equation}
and its time derivative is given by
\begin{equation} \label{eq25_}
    \dot e_{i} = \sum\limits_{j \in N_i} {a_{ij} \left( \dot \eta _i- \dot \eta _j \right) + b_i \left(\dot \eta _i - \dot \eta _i ^d \right)},
\end{equation}
where $a_{ij}$ is a nonnegative constant indicating the communication connections between $i$-th vessel and its neighbor $j$-th vessel ($j \in N_i$), and $b_i$ is also a nonnegative constant that indicates whether or not the $i$-th vessel is permitted to access its desired trajectories, i.e., $\eta _{i}^d$ and its time derivative $\dot{\eta} _i^{d}$; $\Delta_{ij}$ represents the relative pose (position and orientation) between vessels $i$ and $j$, which actually determines the formation shape of a fleet of vessels.

Letting
\begin{align*}
    e &= \left[ { e_{1}^{\rm T}, e_{2}^{\rm T}, \ldots, e_{N}^{\rm T} } \right]^{\rm T}, \dot{e} = \left[ { \dot{e}_{1}^{\rm T}, \dot{e}_{2}^{\rm T}, \ldots, \dot{e}_{N}^{\rm T} } \right]^{\rm T}, \\
    \eta &= \left[ { \eta_{1}^{\rm T}, \eta_{2}^{\rm T}, \ldots, \eta_{N}^{\rm T} } \right]^{\rm T}, \dot{\eta} = \left[ { \dot{\eta}_{1}^{\rm T}, \dot{\eta}_{2}^{\rm T}, \ldots, \dot{\eta}_{N}^{\rm T} } \right]^{\rm T}, \\
    \eta^{d} &= \left[ { \eta_{1}^{d \rm T}, \eta_{2}^{d \rm T}, \ldots, \eta_{N}^{d\rm T} } \right]^{\rm T} \text{and } \dot{\eta}^{d} = \left[ { \dot{\eta}_{1}^{d\rm T}, \dot{\eta}_{2}^{d\rm T}, \ldots, \dot{\eta}_{N}^{d\rm T} } \right]^{\rm T},
\end{align*} 
the time derivative of the consensus formation tracking error of entire vessel system can be expressed as the following compact form
\begin{equation}\label{eq26_}
    \dot{e} = \left( { L+B }\right) \left( { \dot{\eta} - \dot{\eta}^d } \right),
\end{equation}
where matrices $L$ and $B$ are defined previously in Section \ref{s2.1}, describing the communication topology of the formation system considered. Let $v = [v_1,v_2, \ldots, v_N]^{\rm T} $, and by means of the kinematic models of vessels \eqref{eq1a} together with \eqref{eq26_}, the error dynamics for consensus formation tracking is obtained as
\begin{equation}\label{eq27_}
    \dot{e} = \left( { L+B }\right) \left( { J v- \dot{\eta}^d } \right),
\end{equation}
where $J = \text{diag} \left\{J_1,J_2,\ldots, J_N \right\}$. To stabilize above error dynamics into the origin, we may resort to the backstepping design technique and propose the following virtual control law
\begin{equation}\label{eq28_}
    v^{d} = J^{-1} \left( -K_1 e + \dot{\eta}^d   \right),
\end{equation}
where $K_1 \in \mathbb{R}^{6N \times 6N}$ is a positive definite gain matrix to be designed. It is worthwhile noting that since merely the local information is used in this control law \eqref{eq28_}, the proposed virtual controller $v^{d}$ is regarded to be fully distributed. 

Defining an auxiliary variable as 
\begin{equation}\label{eq29__}
    z = v - v^{d},
\end{equation}
together with the proposed virtual control law \eqref{eq28_}, the error dynamics \eqref{eq27_} becomes 
\begin{equation}\label{eq29_}
    \dot{e} = - \left(L+B \right) K_1 e + \left( L+B \right) J z.
\end{equation}
From the knowledge of linear control theory, we can readily conclude that so long as the system matrix $-\left(L+B \right) K_1$ can be made Hurwitz and $z$ is uniformly bounded, all signals in system \eqref{eq29_} is uniformly ultimately bounded, and in particular if $z \to$ 0 as $t \to \infty$, then the origin of the system is a globally exponentially stable equilibrium point. This will also be demonstrated in the section of stability analysis.

To achieve the foregoing purpose, the goal now becomes that finding a control law renders the auxiliary variable $z$ invariant. In this respect, differentiating the auxiliary variable $z$, together with the dynamic models of vessels \eqref{eq7_},  the dynamics of $z$ can be obtained as
\begin{equation}\label{eq30_}
    \dot{z} = \Psi  \theta ^{\star} + \tilde  d - \dot{v}^d,
\end{equation}
where 
\begin{align*}
    \Psi &= \text{diag} \left\{ \Psi _1, \ldots, \Psi _N    \right\}, \quad \theta ^{\star} = \left[  \theta _1^{\star \rm T}, \ldots, \theta _N ^{\star \rm T}  \right]^{\rm T}, \\
    \tilde{d} &= \left[ \tilde{d}_1^{\rm T}, \ldots, \tilde{d}_N^{\rm T} \right]^{\rm T}.
\end{align*}
It should be noted that the control inputs are contained in the regression matrix $\Psi$; for the sake of conciseness, the arguments of regressor $\Psi$ are omitted.

Due to the fact that the accurate model parameters $\theta^{\star}$ are assumed to be unknown in this paper, the dynamic equation \eqref{eq30_} cannot be directly used to synthesize the formation control law. Note that while the real parameter information is unavailable, its estimation values, instead, can be utilized from the developed online learning procedure, i.e., \eqref{eq8a_} and \eqref{eq8b_}. Then, by means of the velocity observer \eqref{eq8a_} and the corresponding definition of observation error, the dynamics of ${v}$ can be equivalently expressed as
\begin{align}\label{eq31_}
    \dot v = \Psi \theta - \dot{\tilde{v}} - \bar L \tilde{v},
\end{align}
where 
\begin{align*}
    \theta &= \left[  \theta _1^{\rm T}, \ldots, \theta _N ^{\rm T}  \right]^{\rm T}, \quad \dot{\tilde{v}} = \left[ \dot{\tilde{v}}_1^{\rm T}, \ldots, \dot{\tilde{v}}_N^{\rm T} \right],\\
    {\tilde{v}} &= \left[ {\tilde{v}}_1^{\rm T}, \ldots, {\tilde{v}}_N^{\rm T} \right],\quad \bar{L} = \text{diag}\left\{ L_1,\ldots, L_N \right\}.
\end{align*}
Consequently, the dynamics of variable $z$ in \eqref{eq30_} can be modified as
\begin{equation}\label{eq32_}
    \dot{z} = \Psi \theta - \dot{\tilde{v}} - \bar L \tilde{v} - \dot{v}^d.
\end{equation}
To facilitate the control design, above expression \eqref{eq32_} is rearranged in the following form
\begin{align}\label{eq33_}
    \dot z  = &- {\bar{C}} (v,\theta) v - {\bar{D}} (v,\theta)v - {\bar{G}}(\eta,\theta) \nonumber \\
    &+ \bar{B}(\theta)\tau - \dot{\tilde{v}} - \bar L \tilde{v} - \dot{v}^d,
\end{align}
where the matrices $\bar{C}$, $\bar{D}$, $\bar{G}$ and $\bar{B}$ are all dependent on the current parameter estimates $\theta$, and $\tau = \left[\tau _1^{\rm T}, \ldots, \tau _N^{\rm T}  \right]^{\rm T}$. The objective now is to seek a control law for $\tau$ such that $z$ can be steered into an invariant set.
\begin{remark}\label{remark5_}
It is common that sliding mode control (SMC) serves as an appropriate robust control technique to realize this requirement. Considering that the severe chattering issue around the sliding mode surface may deteriorate both the control and estimation performance and even render the system unstable, we introduce a neurodynamics model, rather than the employment of sign or saturation function, in the control design so as to obviate the aforementioned drawbacks, and meanwhile we will show that the resulting bioinspired control strategy can still allow for good robust properties.
\end{remark}

As one of the most popular bioinspired neural dynamics, shunting model owing to its desirable characteristics has been extensively used to provide dynamic solutions to various robotic scenarios ranging from path planning to lower-level feedback control for single or even multiple robot systems \cite{xu2020enhanced,BSMC}. The original equation of shunting model for a neuron is given by
\begin{equation}\label{eq34_}
        {\dot \vartheta _i} =  - {a_i}{\vartheta_i} + \left( {{b_i} - {\vartheta _i}} \right)z_i^ +  - \left( {{d_i} + {\vartheta_i}} \right)z_i^ - ,
\end{equation}
where $z_i ^+$, $z_i ^- \in \mathbb{R}$ represent the environmental excitatory and inhibitory signals applied on the $i$-th neuron, respectively; $\vartheta _i \in \mathbb{R}$ represents the neural activity of $i$-th neuron; $a_i$, $b_i$ and $d_i$ are positive real constants associated.
\begin{remark}
As we can see from the above shunting equation, the variable $\vartheta_i$ exhibits a dynamic behavior to the environmental changes, i.e., $z_i ^+$ and $z_i ^-$, which means that it can be used to provide a more consistent behavior even when faced with the environmental disturbances and noises. In addition, the state of $\vartheta _i$ is bounded upper by $b_i$ and lower by $-d_i$. In what follows, it will be shown that the controller aided with shunting model is able to produce improved control activities over conventional SMC schemes.
\end{remark}

The shunting model \eqref{eq34_} is given by the scalar form, and we may extend it to a higher dimension. Let 
 \begin{align*}
     &\vartheta = \left[ \vartheta _1^{\rm T}, \ldots, \vartheta _N^{\rm T} \right] ^{\rm T}, \quad \vartheta_i \in \mathbb{R}^{6},\\
     &z = \left[ z _1^{\rm T}, \ldots, z _N^{\rm T} \right] ^{\rm T}, \quad z_i \in \mathbb{R}^{6},\\
     & \bar{g} \left( z \right) =  \left[ \bar g_1 ^{\rm T}\left( z_1 \right), \ldots, \bar g _N ^{\rm T}\left( z_N \right)    \right] ^{\rm T}, \quad \bar{g}_i \in \mathbb{R}^{6}.
 \end{align*}
 Note that the subscript $i$ here denotes the $i$-th component of a vector and $i \in \Gamma$. Then, the higher dimensional shunting model can be represented as
 \begin{equation} \label{eq35_}
     \dot \vartheta = - \Lambda  \vartheta + \bar g \left( z \right),
 \end{equation}
 where
\begin{align} \label{eq36_}
    & \Lambda = \text{diag} \left\{ a_1 \text{I}_{6\times 6} + \Xi \left( z_1 \right), \ldots, a_N \text{I}_{6\times 6} + \Xi \left( z_N \right) \right\}, \nonumber \\
    & \Xi \left( z_i \right)= \text{diag} \left\{ \left| {z_{i,1}} \right|, \ldots, \left| {z_{i,6}} \right|  \right\} , \quad i \in \Gamma, \nonumber \\
    & \bar g _i\left( z_i \right) =  \left[ g _i\left( z_{i,1} \right), \ldots, g _i\left( z_{i,6} \right) \right] ^{\rm T}, \quad i \in \Gamma, \nonumber \\
    & g_i \left( z_{i,j} \right) = \begin{cases}
    b_i z_{i,j}, & z_{i,j} \ge 0,   \\
    d_i z_{i,j}, & z_{i,j} < 0.
    \end{cases} 
\end{align}
Here, the adjustable parameters $a_i$, $b_i$, and $d_i \ (i \in \Gamma)$ are the positive constants associated with the  model.

Integrated with \eqref{eq35_} and \eqref{eq36_}, the following bioinspired control law is proposed to stabilize the $z$-subsystem  $\eqref{eq33_}$
\begin{align} \label{eq37_}
    \tau &= \bar B ^{-1} \left[ \dot{v}^d + \bar{C} v + \bar{D} v +\bar{G} -K_2 \vartheta   \right],
\end{align}
where  $K_2 \in \mathbb{R}^{6N \times 6N}$ is a positive define gain matrix to be designed. The realization of the proposed distributed formation control protocol is illustrated in the \textbf{Algorithm} $\ref{alg:2}$.

\begin{figure}[!t]
\label{alg:1}
\renewcommand{\algorithmicrequire}{\textbf{Input:}}
\renewcommand{\algorithmicensure}{\textbf{Output:}}
\removelatexerror
	\begin{algorithm}[H]
		\caption{Distributed Bioinspired Robust Learning-Based Formation Control Algorithm.}
		\label{alg:2}
		\begin{algorithmic}[1]
			\STATE {For each AUV $i$, $i=1, \ldots, N$:}
			\STATE {Initialize the controller: choose suitable values for $L_i$, $P_i$, $K_{1,i}$, $K_{2,i}$, $a_i$, $b_i$ and $c_i$; set the initial states appropriately both for the learning procedure and shunting model.}
			\STATE {\textbf{while} The formation objective is not complete \textbf{do}}
			\STATE {\quad Sample the system states $\eta _i$ and $v_i$;}
			\STATE {\quad Calculate current estimates for $\hat v_i$ and $\theta _i$ using the \\ \quad adaptation law given in \eqref{eq8a_} and \eqref{eq8b_};}
			\STATE {\quad Receive the neighbors' information $\eta _j$ and desired \\ \quad  trajectories $\eta ^d$ and $\dot \eta ^d$ as applicable; }
			\STATE {\quad Apply the current control input $\tau _i$ calculated by the \\ \quad control law presented in \eqref{eq35_}--\eqref{eq37_}. } 
			\STATE {\textbf{end while}}
 
		\end{algorithmic}
	\end{algorithm}
\end{figure}

\begin{remark}\label{remark_6}
It can be observed that the proposed controller is implemented in a fully distributed way, and besides the control law is consisted of two portions, i.e., the learning-based equivalence control and bioinspired control. To be specific, in order to counteract the nonlinearities and uncertainties in the vessels' dynamic model, the learning-based equivalence control is designed on the top of the parameter estimators, i.e., \eqref{eq8a_} and \eqref{eq8b_}, where the system matrices $\bar{B}$, $\bar{C}$, $\bar{D}$ and $\bar{G}$ are updated in a real-time fashion. Moreover, the bioinspired control term is synthesized with the aim to provide a smooth and practical control effort and, at the same time, stabilize the subsystem of $z$ even in the presence of estimation errors.
\end{remark}

\subsection{Stability analysis}
The input-to-state stability of the proposed learning-based bioinspired control scheme is proven in this section. To this end, plugging the proposed control law \eqref{eq37_} into the equation \eqref{eq33_} together with \eqref{eq29_}, we obtain the following closed-loop system
\begin{align} 
    \dot{e} &= - \left(L+B \right) K_1 e + \left(L+B \right)Jz,  \label{eq38_} \\
    \dot z &= - K_2 \vartheta - \dot{\tilde{v}} - \bar{L}\tilde{v}, \label{eq39_} \\
    \dot \vartheta &= - \Lambda  \vartheta + \bar g \left( z \right).  \label{eq40_}
\end{align}
Notice the fact that the resulting closed-loop system is made up of three subsystems \eqref{eq38_}--\eqref{eq40_}; in particular, $e$-subsystem is cascaded with the $z$-subsystem by viewing $z$ as the input, and subsystems of $z$ and $\vartheta$ are interconnected. To facilitate the analysis, letting $\xi = \left[ z^{\rm T}, \vartheta^{\rm T} \right]^{\rm T}$ and $\delta =  - \dot{\tilde{v}} - \bar{L}\tilde{v}$, the subsystems \eqref{eq39_} and \eqref{eq40_} can be rewritten in a more compact form as
\begin{align} 
    \dot \xi &= T \xi + N\delta, \label{eq41_}
\end{align}
where
\begin{align*}
    & T =  \begin{bmatrix} 0 & -K_2  \\
    \bar{G} & -\Lambda  \end{bmatrix}, 
    \quad N = \begin{bmatrix} \textbf{1} \\
    \textbf{0} \end{bmatrix}.
\end{align*}
Here, it follows from the property of function $\bar{g}(z)$ that the matrix $\bar{G}$ is diagonal and each entry in its diagonal takes value of either $b_i$ or $d_i \ (i \in \Gamma)$, both of which are positive constants. As a result, $\bar{G}$ is a positive definite diagonal matrix.

We provide the following theorem to establish the input-to-state stability of the resulting closed-loop system with the proposed distributed learning-based bioinspired formation control protocol \eqref{eq35_}, \eqref{eq37_}.
\newtheorem{theorem}{Theorem} 
\begin{theorem} \label{theorem1}
The system \eqref{eq38_}--\eqref{eq40_} is input-to-state stable if matrices $K_1$, $K_2$, $\Lambda$ and $\bar{G}$ are chosen properly such that the matrices $-\left(L+B\right) K_1$ and $T$ are both Hurwitz.
\end{theorem}

\begin{proof}
 Utilizing the cascaded interconnection of subsystems \eqref{eq38_} and \eqref{eq41_}, the proof may proceed with two steps: first step shows the input-to-state stability of the $e$-subsystem with respect to $z$, and the second step tries to show that the $\xi$-subsystem is input-to-state stable as well regarding the $\delta$.

\textit{Step1: Input-to-state stability of $e$-subsystem.}

Let $\bar z = \left(L+B\right) J z$. It is observed that the $e$-subsystem \eqref{eq38_} is a linear-time-invariant (LTI) system enforced by the input $\bar z$, and according to the condition that $-(L+B)K_1$ is designed to be Hurwitz, the solution of such a LTI system can be readily given by
\begin{align}\label{eq42_}
     e\left( t \right) = {e^{ - \left( {L + B} \right)K_1 t}}{ e}\left( 0 \right) + \int_0^t {{e^{ - \left( {L + B} \right)K_1\left( {t - \tau } \right)}} \bar{z}(\tau)d\tau } .
\end{align}
Applying the inequality $\left\| e^{-(L+B)K_1 t}  \right\| \le k_1 e^{-\alpha _1 t}   $, where $k_1$ and $\alpha _1$ are some positive constants, yield
\begin{align} \label{eq43_}
    \left\| { e}\left( t \right) \right\| & \le k_1 e^{-\alpha _1 t} \left\|  e\left( 0 \right) \right\|  + \int_0^t {{k_1 e^{ - \alpha _1 \left( {t - \tau } \right)}} \bar{z}(\tau)d\tau } \nonumber \\
    & \le  k_1 e^{-\alpha _1 t} \left\| { e}\left( 0 \right) \right\|  + \frac{k_1}{\alpha _1} \sup_{0\le \tau \le t} \left\|  \bar{z}(\tau) \right\| \nonumber \\
    & =  k_1 e^{-\alpha _1 t} \left\| { e}\left( 0 \right) \right\|  + \frac{k_1 k_2}{\alpha _1} \sup_{0\le \tau \le t} \left\|  {z}(\tau) \right\|,    
\end{align}
where $k_2$ is the maximal eigenvalue of matrix $L+B$. It shows from inequality \eqref{eq43_} that the trajectories of subsystem \eqref{eq38_} is bounded whenever the signal $z(t)$ is bounded. This also demonstrates that above subsystem is of input-to-state stability with respect to $z(t)$.

\textit{Step2: Input-to-state stability of $\xi$-subsystem.}

Since matrix $T$ is Hurwitz, then there exists a symmetric positive define matrix $P$ such that 
\begin{align}
    T^{\rm T} P + P T = -{I}, \label{eq44_}
\end{align}
where ${I}$ is the identity matrix.

Propose the following Lyapunov function candidate
\begin{align}
    V_2 = \xi ^{\rm T} P \xi. \label{eq45_}
\end{align}
Taking the time derivative of $V_2$ along the trajectories of $\xi$-subsystem, yield
\begin{align}
    \dot{V_2} &= \dot{\xi}^{\rm T} P \xi + \xi ^{\rm} P \dot{\xi} \nonumber\\
    &= \left( T\xi + N \delta \right)^{\rm T} P \xi + \xi ^{\rm T} P \left( T\xi + N \delta \right) \nonumber \\
    &= -\xi^{\rm T} \xi + \left( N\delta \right)^{\rm T} P \xi + \xi ^{\rm T} P N \delta \nonumber \\
    & \le - \left\| \xi \right\| ^2 +2 \left\| N \right\| \left\| \delta \right\| \left\| P \right\| \left\| \xi \right\| \nonumber \\
    & \le - \left\| \xi \right\| ^2 +2\left\| P \right\| \left\| \delta \right\|  \left\| \xi \right\|. \label{eq46_}
\end{align}
As the result of Lemma \ref{lemma2_}, there exists a positive number $\gamma$ such that the following inequality holds
\begin{align}
    \left\| \delta \right\| = \left\| - \dot{\tilde{v}} - \bar{L}\tilde{v} \right\| \le \left\| \dot{\tilde{v}}\right\| + \left\| \bar{L}\right\| \left\| \tilde{v} \right\| \le \gamma. \label{eq47_}
\end{align}
Thus, we may have
\begin{align}
    \dot{V}_2 &\le - \left\| \xi \right\| ^2 + 2 \gamma \left\| P \right\| \left\| \xi \right\| \nonumber \\
    & \le - \left( 1 -\kappa \right) \left\| \xi \right\| ^2, \ \text{whenever } \left\| \xi \right\| \ge \frac{2\gamma }{\kappa}\left\| P\right\|, \label{eq48_}
\end{align}
where $0 < \kappa < 1$. Letting $\beta = \left( 2\gamma \left\| P \right\| \right) / \kappa$, from \eqref{eq48_} together with \eqref{eq45_} we may obtain
\begin{align}
    V_2 \le \lambda_{\max} \left( P \right) \beta^2, \label{eq49}
\end{align}
and  furthermore,
\begin{align}
    \left\| \xi \right\| \le \sqrt{\frac{\lambda_{\max} \left( P \right)}{\lambda_{\min} \left( P \right)}} \beta, \label{eq50_}
\end{align}
where $\lambda_{\max} \left( \cdot\right)$ and $\lambda_{\min} \left( \cdot\right)$ denote the maximum and minimum eigenvalues of a matrix, respectively. The ultimate bound of $\xi$ is given by \eqref{eq50_}, which shows that the $\xi$-subsystem is input-to-state stable. Together with the input-to-state stability property of $e$-subsystem obtained from the \textit{Step1}, we can conclude that the closed-loop system \eqref{eq38_}--\eqref{eq40_} is input-to-state stable.  This completes the proof.\end{proof}

\begin{remark}
Note that it is easy to verify that $-(L+B)K_1$ is Hurwitz if $K_1$ is positive diagonal, which is attributed to the fact that $L+B$ is positive definite. For the Hurwitz property of $T$, observe that $K_2$, $\bar{G}$ and $\Lambda$ are all diagonal matrices, and thus the system \eqref{eq41_} represents a batch of mutually independent 2nd-order subsystems. Hence, the Hurwitz property can be established by letting all the eigenvalues of such subsystems have negative real parts, and in particular the analytical solutions of eigenvalues of a 2nd-order system can be easily obtained.
\end{remark}


\begin{remark} \label{remark7_}
It can be shown from Theorem \ref{theorem1} that the proposed formation control system is of good robustness in rejecting various unknown disturbances. To be more specific, the unavoidable modeling uncertainties first are addressed actively by the online learning procedure where the system dynamic parameters are identified in a real-time manner. Subsequently, the effects of the estimation error remained can be further counteracted by the proposed robust controller, in which the high-gain strategy is circumvented and the resulting control activities are much smoother when compared to the SMC-based approaches.
\end{remark}

\section{SIMULATION RESULTS}\label{s5}
\begin{figure}[!t]
\centering
\includegraphics[width=0.3\textwidth,trim=10 30 10 10,clip]{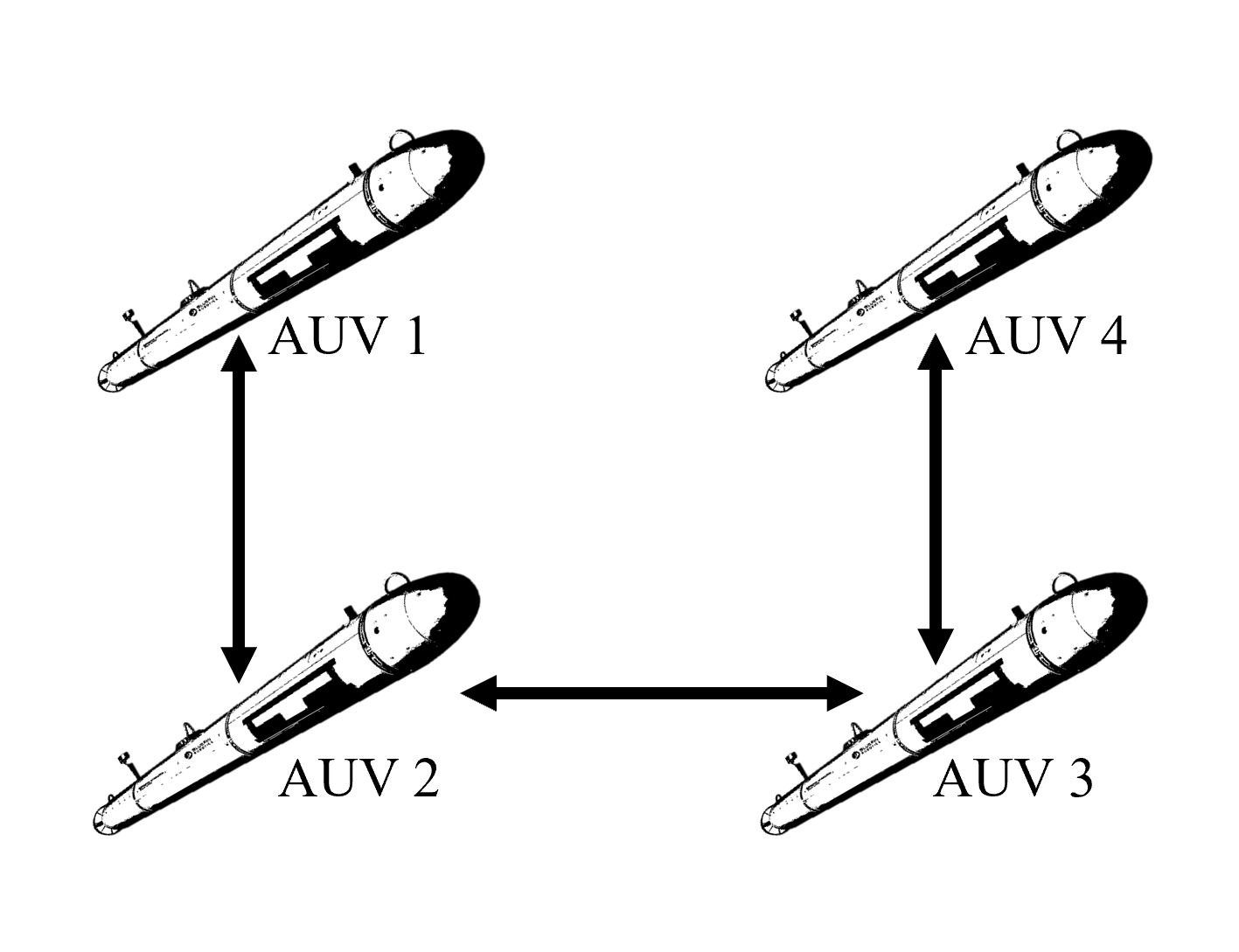}
\caption{The communication topology graph for the consensus formation tracking of 4 AUVs.}
\label{fig_2}
\end{figure}

\begin{figure*}[!htbp] 
	\centering
	\subfloat[]{\includegraphics[width=0.3\textwidth]{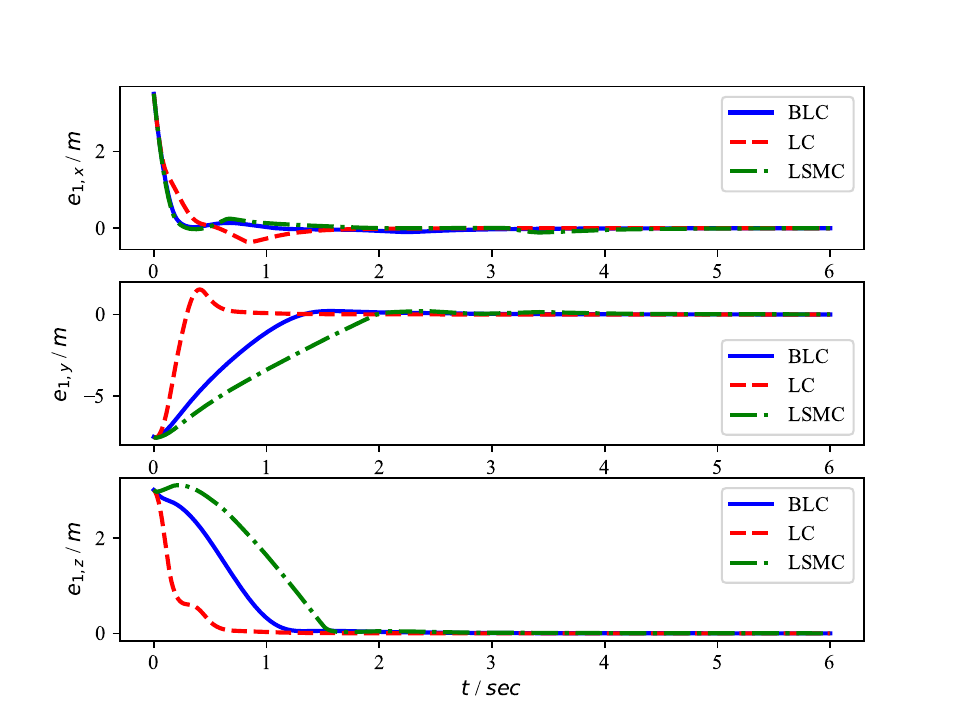}}\hspace{0.0in} 
	\subfloat[]{\includegraphics[width=0.3\textwidth]{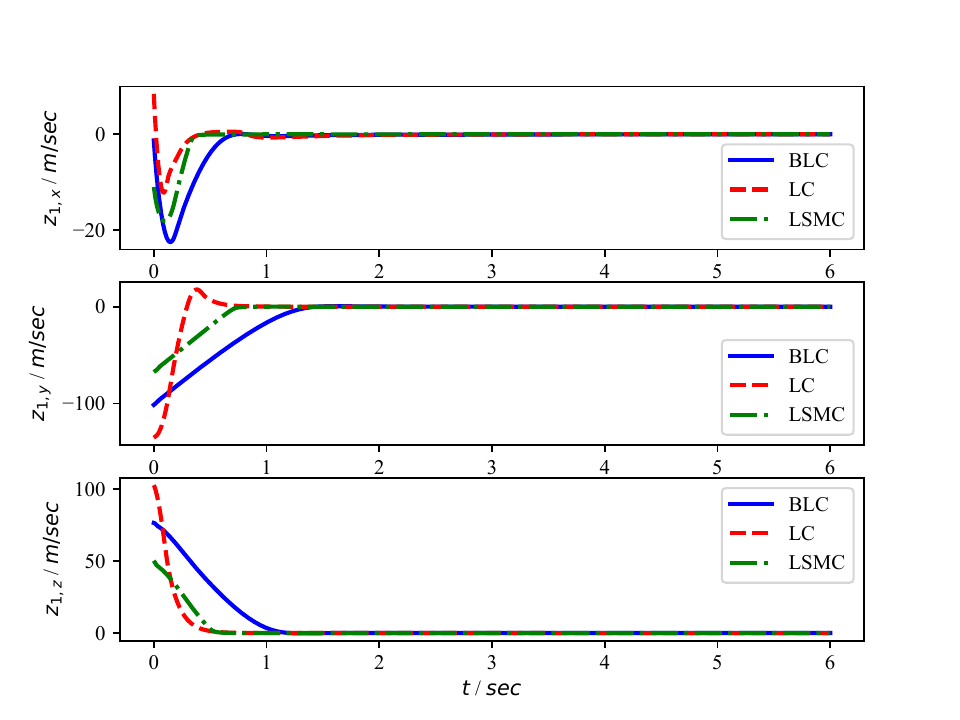}}
	\subfloat[]{\includegraphics[width=0.3\textwidth]{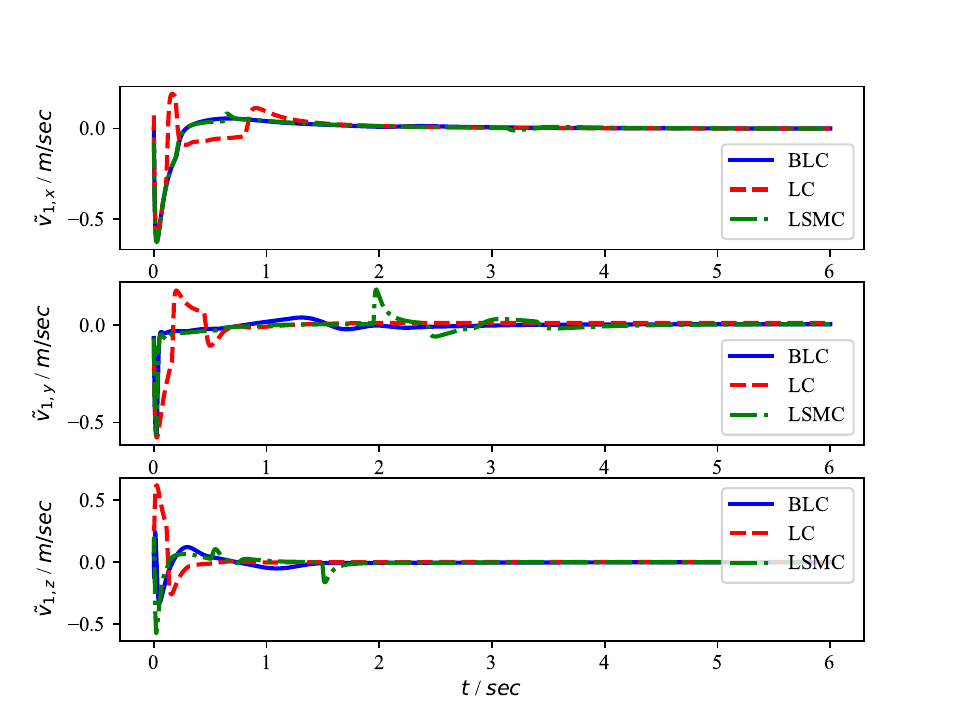}}
	\caption{The adaptive formation tracking performance of AUV 1 under three types of controllers. (a) The consensus formation tracking error. (b) Trajectory of auxiliary variable $z_1$. (c) The observation error.}
	\label{fig_3}
\end{figure*} 

\begin{figure*}[!htbp] 
	\centering
	\subfloat[]{\includegraphics[width=0.3\textwidth]{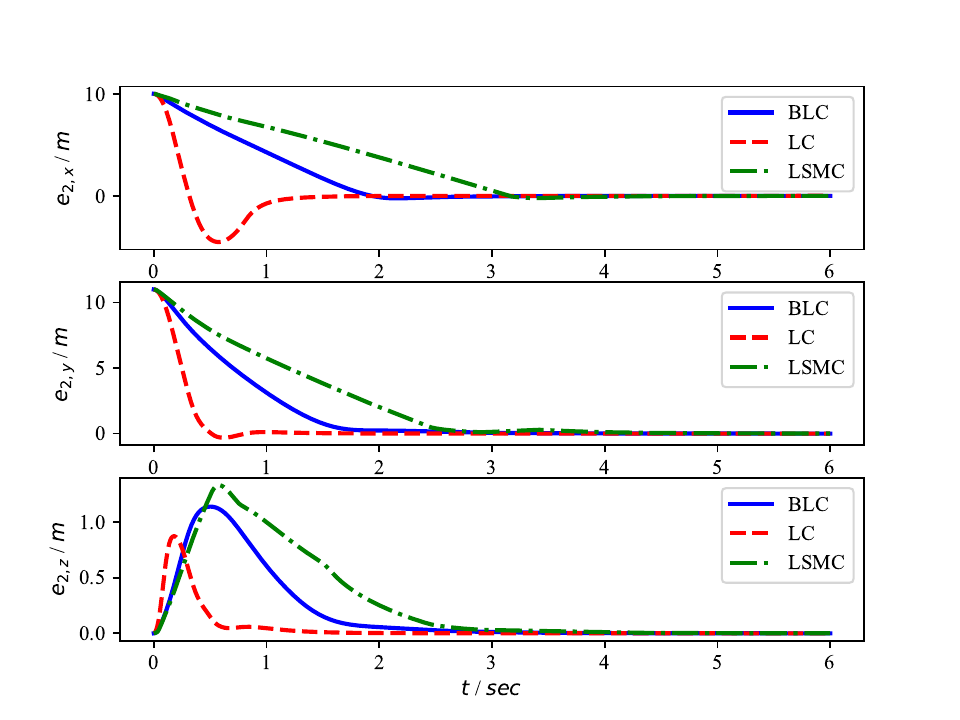}}\hspace{0.0in} 
	\subfloat[]{\includegraphics[width=0.3\textwidth]{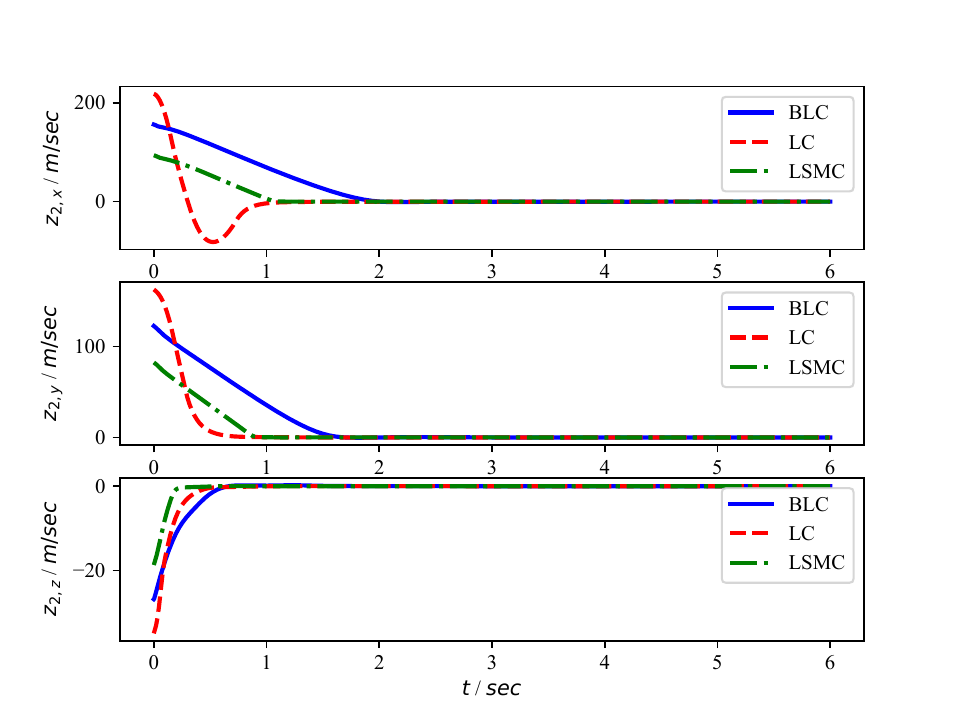}}
	\subfloat[]{\includegraphics[width=0.3\textwidth]{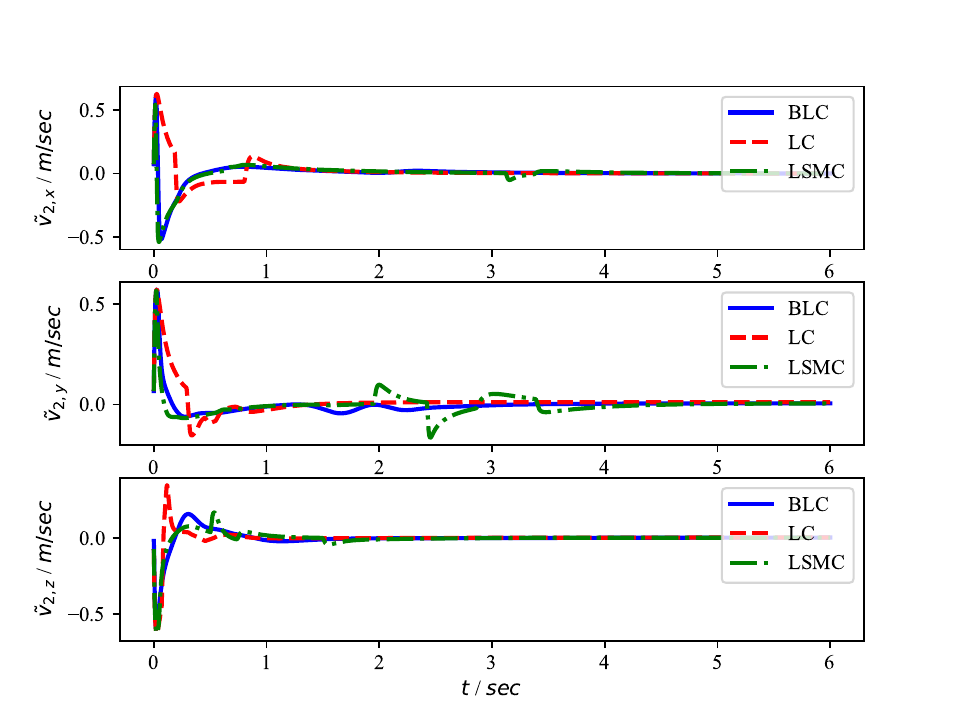}}
	\caption{The adaptive formation tracking performance of AUV 2 under three types of controllers. (a) The consensus formation tracking error. (b) Trajectory of auxiliary variable $z_2$. (c) The observation error.}
	\label{fig_4}
\end{figure*} 

\begin{figure*}[!htbp] 
	\centering
	\subfloat[]{\includegraphics[width=0.3\textwidth]{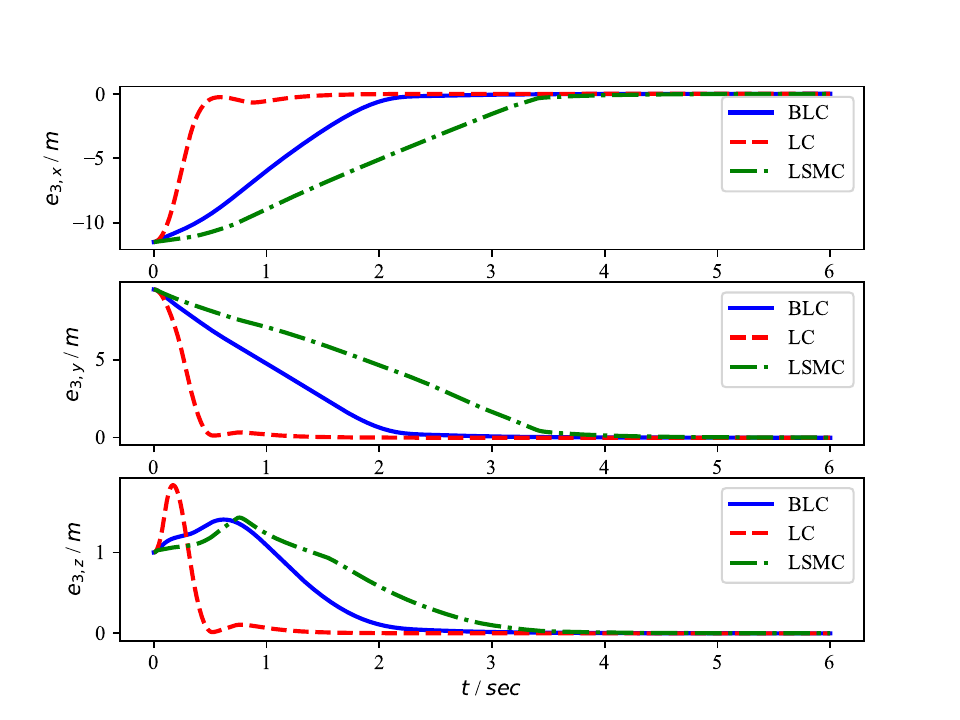}}\hspace{0.0in} 
	\subfloat[]{\includegraphics[width=0.3\textwidth]{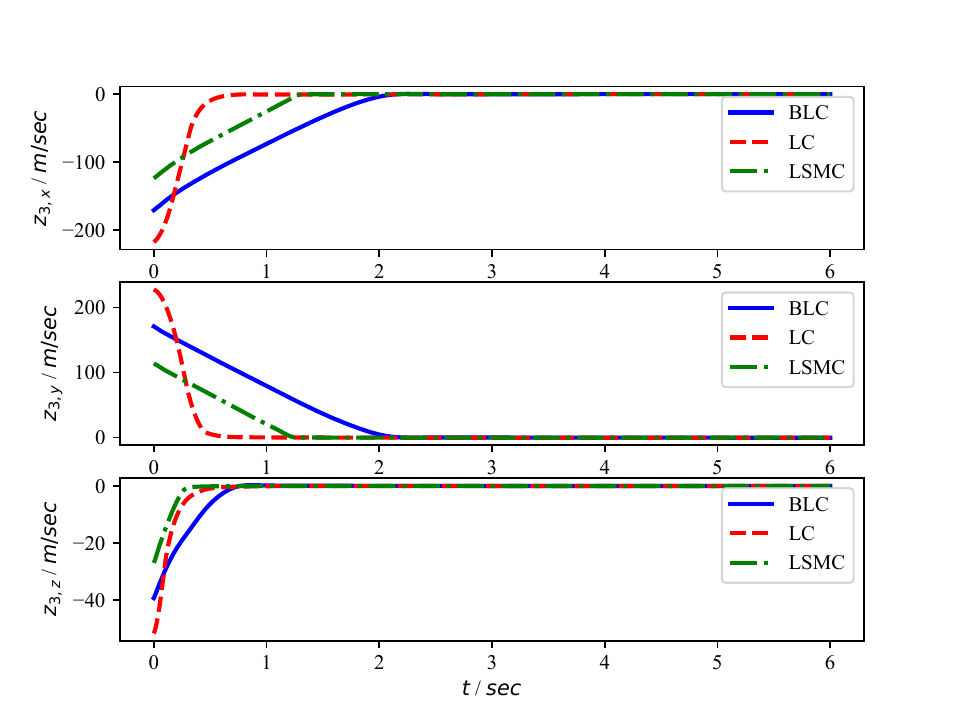}}
	\subfloat[]{\includegraphics[width=0.3\textwidth]{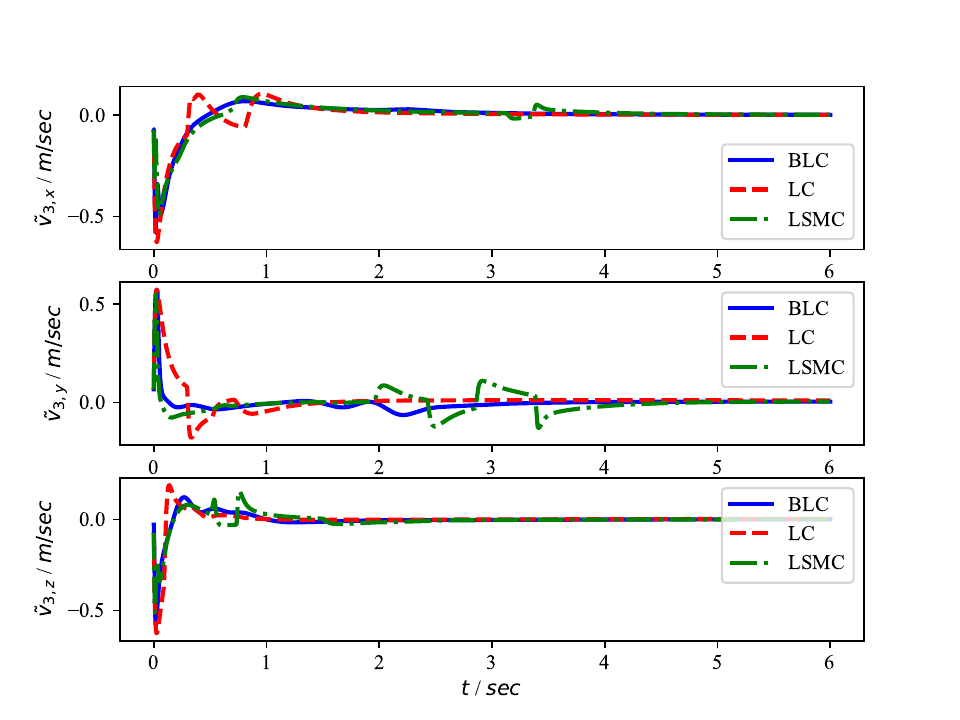}}
	\caption{The adaptive formation tracking performance of AUV 3 under three types of controllers. (a) The consensus formation tracking error. (b) Trajectory of auxiliary variable $z_3$. (c) The observation error.}
	\label{fig_5}
\end{figure*} 

\begin{figure*}[!htbp] 
	\centering
	\subfloat[]{\includegraphics[width=0.3\textwidth]{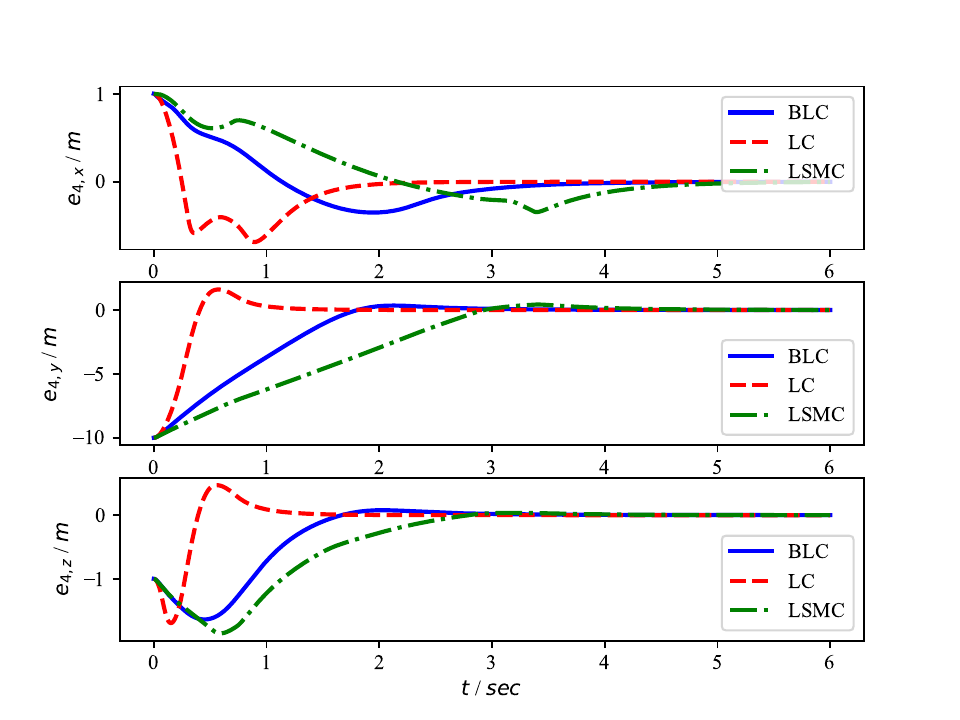}}\hspace{0.0in} 
	\subfloat[]{\includegraphics[width=0.3\textwidth]{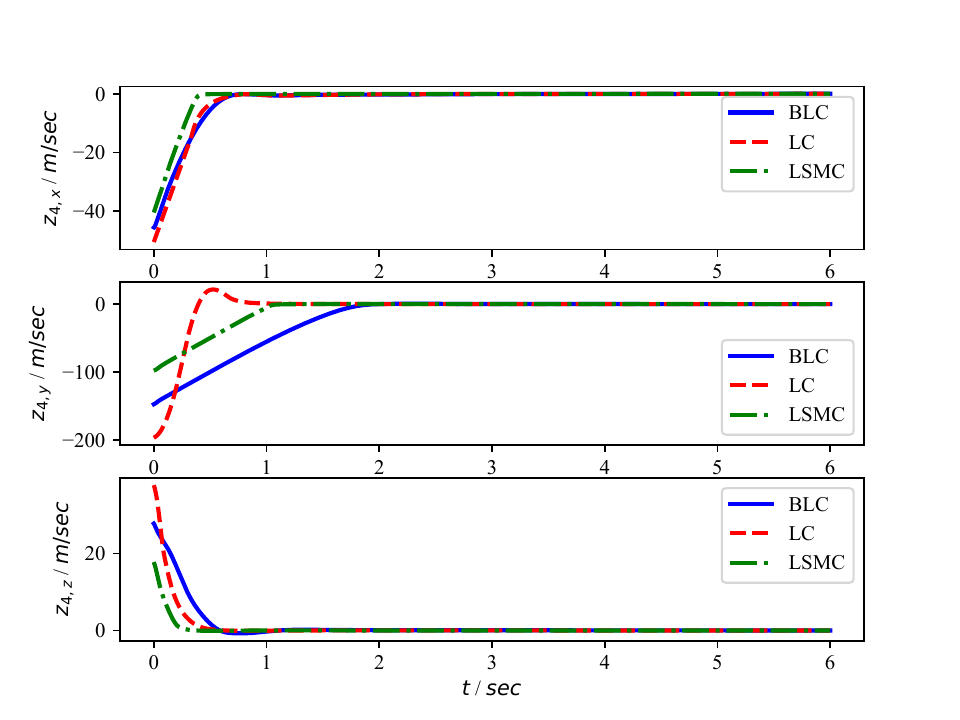}}
	\subfloat[]{\includegraphics[width=0.3\textwidth]{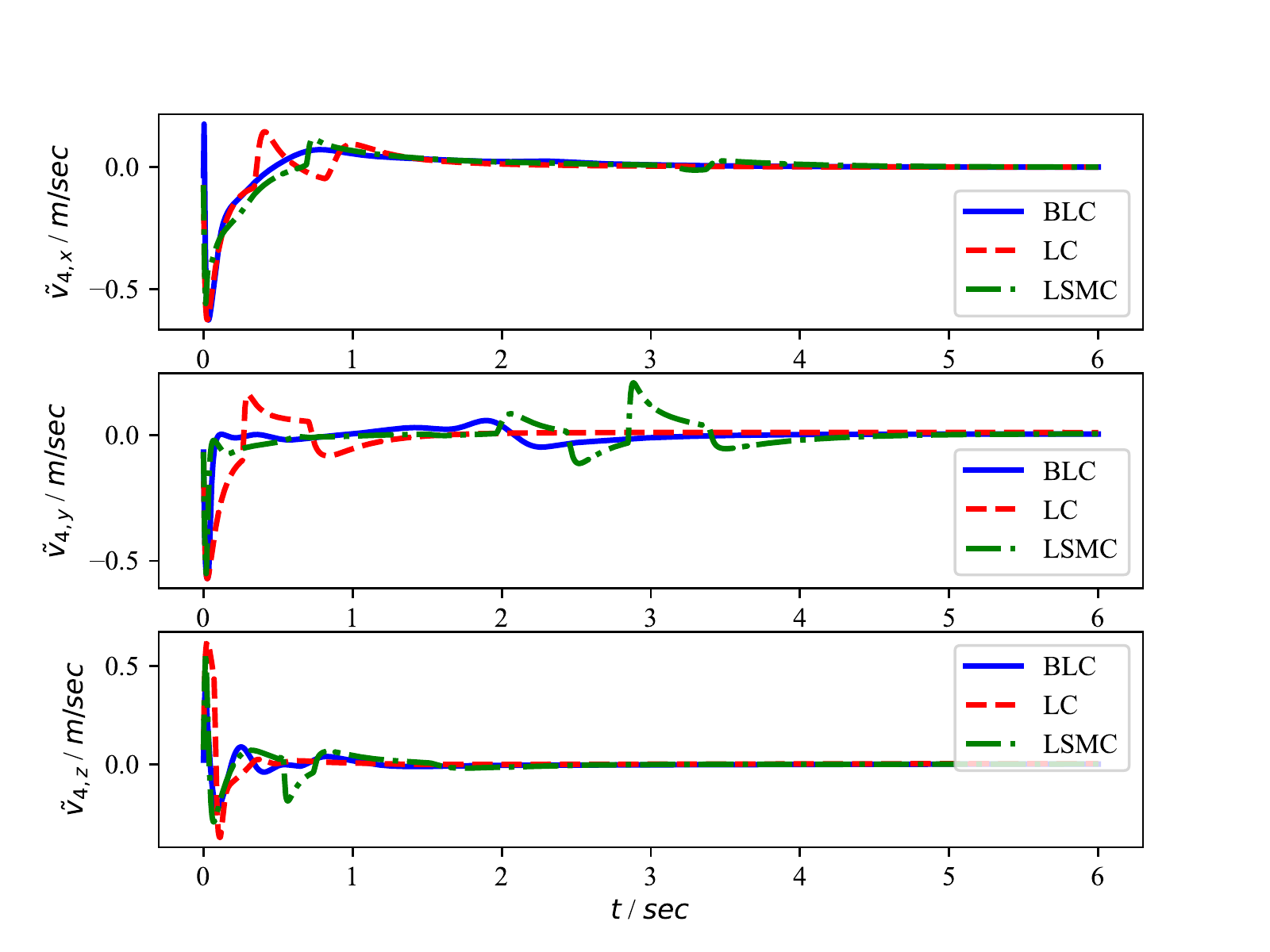}}
	\caption{The adaptive formation tracking performance of AUV 4 under three types of controllers. (a) The consensus formation tracking error. (b) Trajectory of auxiliary variable $z_4$. (c) The observation error.}
	\label{fig_6}
\end{figure*} 
To validate the efficiency and superiority of the proposed distributed formation tracking protocol, numerous simulation experiments are conducted and compared in this section, where two types of commonly used nonlinear controllers, i.e, backstepping control and sliding mode control, are adopted as the baselines to illustrate the formation performances in three different scenarios in terms of the formation tracking accuracy, disturbance rejection, and the noise suppression. In all of the simulation cases, four underwater vessels are used to construct a formation system, and each vessel is steered by its own embedded formation controller whose objectives are to form a prescribed formation shape, i.e., a quadrilateral geometry profile, and meanwhile follow a desired straight line trajectory in 3-dimensional space. The communication topology among the vessels of the considered formation system is illustrated in Fig. \ref{fig_2}.

The dynamics of vessels employed is described by the equations \eqref{eq1a} and \eqref{eq1b}, and the system parameters associated with the dynamic equations are given as follows with international units: $m_i = 25$, $I_{x,i}=25$, $I_{y,i}=20$, $I_{z,i}=30$, $\beta _{vx,i} = -10$, $\beta _{vy,i} = -8$, $\beta _{vz,i} = -12$, $\beta _{\dot vx,i} = -8$, $\beta _{\dot vy,i} = -6$, $\beta _{\dot vz,i} = -8$, $\beta _{\omega x,i} = -0.35$, $\beta _{\omega y,i} = -0.2$, $\beta _{\omega z,i} = -0.25$, $\beta _{\dot \omega x,i} = -25$, $\beta _{\dot \omega y,i} = -35$, $\beta _{\dot \omega z,i} = -30$, $\left( {i \in \left\{ {1,2,3,4} \right\}}  \right)$. Note that these parameters are just used to simulate the dynamic process of the vessels, and are unavailable for the controller design. In other words, all of the distributed formation controllers used in the simulations are additionally equipped with a learning procedure (developed in Section \ref{s3}) to provide a real-time parameter estimation. The weights on the communication topological graph are set as ${a_{12}} = {a_{21}} = {a_{23}} = {a_{23}} = {a_{34}} = {a_{43}} = 1$, and since it is assumed that all of the vessels are allowed to access the information of the desired trajectory, we set ${b_1} = {b_2} = {b_3} = {b_4} = 1$. To generate a prescribed formation profile, the corresponding relative positions between vessels are determined as  ${\delta _{12}} = {\left[ {0,10,0} \right]^{\rm T}}$, ${\delta _{21}} = {\left[ {0, - 10,0} \right]^{\rm T}}$, ${\delta _{23}} = {\left[ { - 10,0,0} \right]^{\rm T}}$, ${\delta _{32}} = {\left[ {10,0,0} \right]^{\rm T}}$, ${\delta _{34}} = {\left[ {0, - 10,0} \right]^{\rm T}}$ and ${\delta _{43}} = {\left[ {0,10,0} \right]^{\rm T}}$. Additionally, the desired trajectory to be followed is given by $\eta _1^d\left( t \right) = {\left[ {30 - 30{e^{ - t}},5t,2t} \right]^{\rm T}}$, and the vessels' postures are expected to align to $\eta _2^d\left( t \right) = {\left[ {0,0,0} \right]^{\rm T}}$. The initial conditions of the four vessels are set as ${\eta _1(0)} = {\left[ {3,3,3,0.3,0,0.2} \right]^{\rm T}}$, ${\eta _2(0)} = {\left[ {2.5,3.5,3,0.2,0,0.25} \right]^{\rm T}}$, ${\eta _3(0)} = {\left[ {2,3,3,0.3,0,0.2} \right]^{\rm T}}$, ${\eta _4(0)} = {\left[ {3,3,2,0.3,0,0.2} \right]^{\rm T}}$, and ${v_i(0)} = {\mathbf{0}_{6 \times 1}},\left( {i \in \left\{ {1,2,3,4} \right\}} \right)$. 

The two baseline controllers used for comparison (i.e., learning-based backstepping control (LC) and learning-based sliding mode control (LSMC)) are given, respectively, as follows
\begin{align} 
    &\tau_{lc} = \bar B ^{-1} \left[ \dot{v}^d + \bar{C} v + \bar{D} v +\bar{G} -K_2 z   \right], \label{eq51_} \\
    &\tau_{lsmc} = \bar B ^{-1} \left[ \dot{v}^d + \bar{C} v + \bar{D} v +\bar{G} -K_2 s   \right],  \label{eq52_}
\end{align}
where parameter matrices $\bar B$, $\bar C$, $\bar D$, and $\bar G$ are all obtained with the on-line learning procedure \eqref{eq8a_} and \eqref{eq8b_} for both controllers; $v^{d}$ and $z$ are given by \eqref{eq28_} and \eqref{eq29__}, respectively, and the sliding mode variable used in \eqref{eq52_} is defined as
\begin{align}
    s = \text{sat} (z). \label{eq53_}
\end{align}
The control parameters used in the simulations are listed in TABLE \ref{t1},  and for convenience the proposed distributed bioinspired learning-based formation control protocol is shorten as the BLC scheme.
\begin{table*}[!htbp] 
\centering
\caption{Control parameters}
\begin{tabular}{cccc}   
\toprule 
Parameters & BLC & LC & LSMC \\
\midrule  
$L_i$ & diag(100,100,100,100,100,100)& diag(100,100,100,100,100,100) & diag(100,100,100,100,100,100)\\
$P_i$ & diag(0.1,0.1,0.1,0.1,0.1,0.1)& diag(0.1,0.1,0.1,0.1,0.1,0.1) & diag(0.1,0.1,0.1,0.1,0.1,0.1)\\
$K_{1,i}$ & diag(15,15,15,5,5,5) & diag(25,25,25,5,5,5) & diag(15,15,15,5,5,5)\\
$K_{2,i}$ & diag(1,1,1,0.5,0.5,0.5)&  diag(10,10,10,5,5,5)& diag(60,60,60,15,15,15) \\
$a_i$ & 10  & N/A & N/A\\
$b_i$ & 50 & N/A & N/A\\
$d_i$ & 50 & N/A & N/A\\
\bottomrule 
\end{tabular}
\label{t1}
\end{table*}

\begin{figure}[!htbp]
\centering
\includegraphics[width=0.4\textwidth]{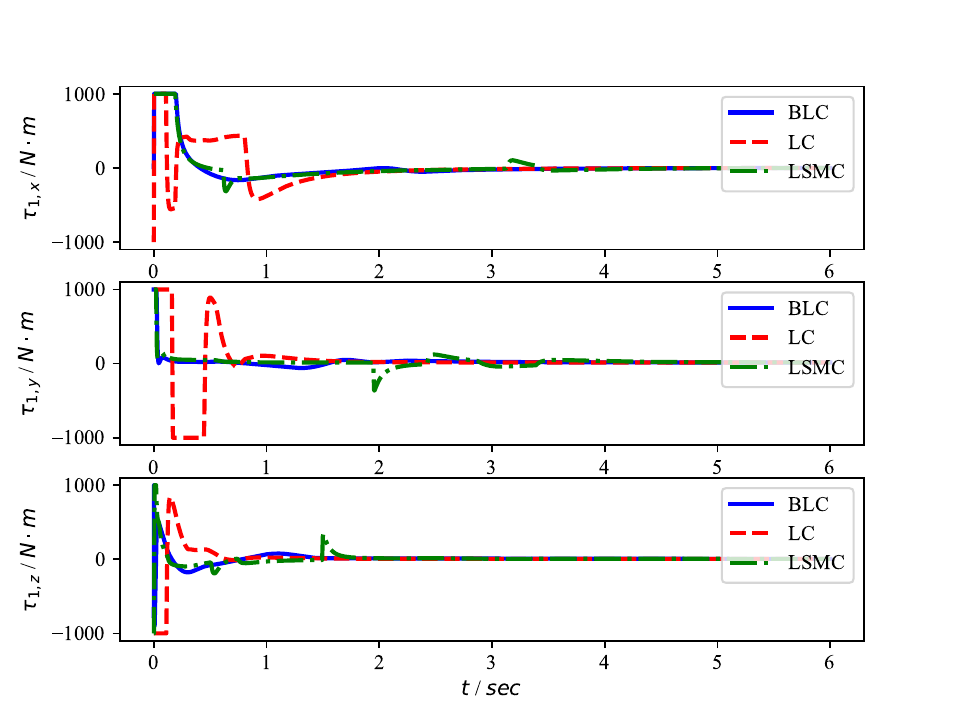}
\caption{The control signals of AUV 1 generated by the three types of controllers.}
\label{fig_7}
\end{figure}

\begin{figure}[!htbp]
\centering
\includegraphics[width=0.4\textwidth,trim=50 30 50 50,clip]{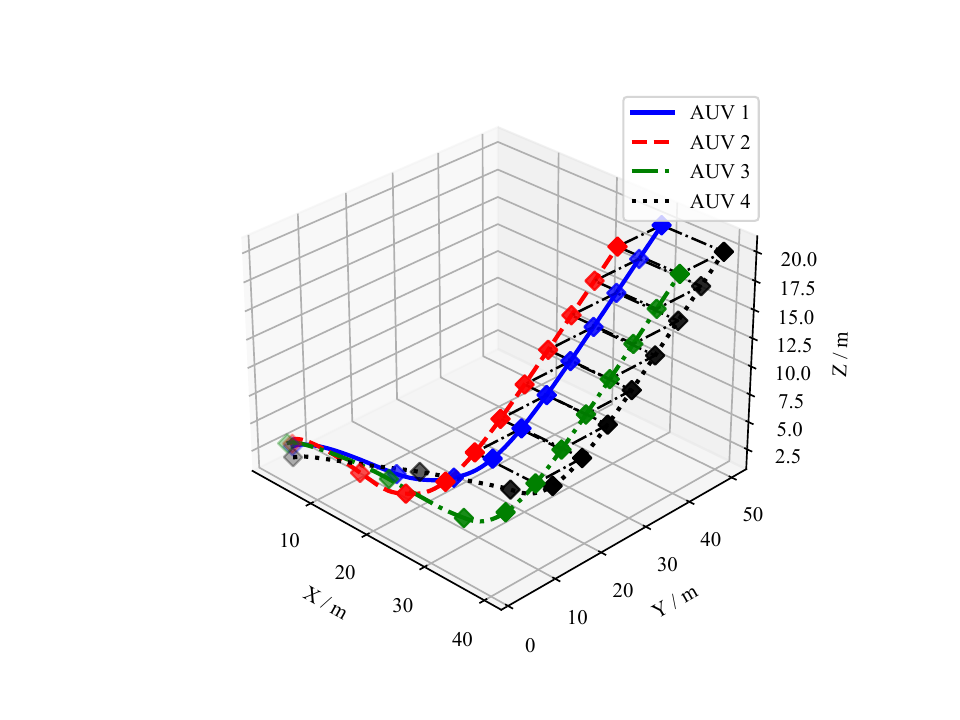}
\caption{The motion scene of overall formation system under BLC protocol.}
\label{fig_7_}
\end{figure}

In the first scenario, we compare the adaptive formation tracking performance of three types of distributed control protocol, i.e., BLC (proposed approach), LC, and LSMC, without applying disturbances and noises. It is illustrated by Figs. \ref{fig_3}--\ref{fig_6} that all three formation control protocols achieve the adaptive formation tracking objectives. In other words, the consensus formation tracking errors and introduced auxiliary variables of four vessels are all enforced to the zeros under the control activities, and in addition, the observation errors of the AUVs are all brought to zeros as well, indicating that the on-line learning procedures are in effect and able to provide real-time parameter identifications. Furthermore, it is observed evidently that the proposed bioinspired approach shows a more moderate performance over both the backstepping approach and sliding mode scheme. In particular, the LC approach behaves more aggressively, which necessitates relatively large velocity commands as indicated by the evolution of the auxiliary variables, and besides, more overshoots can be found in the entire control process. On the other hand, while the gain matrix $K_2$ of the LSMC scheme is deliberately tuned small enough to mitigate the chattering issue, the spikes still appear in the learning process in all of the vessels due to the employment of switching-like control law. The above statements can also be justified by checking the control activities of AUV 1 as shown in Fig. \ref{fig_7} (other AUVs' are pretty much similar), in which much more control efforts are used in the LC approach, while the faster convergence speed can be obtained, however, resulting in a more oscillating and unsmooth behavior. Similar to the LC strategy, the control of the LSMC scheme also exhibits an unsmooth behavior. In contrast, the BLC approach behaves more moderately and reasonably among all three types of control strategies, namely, less oscillation, low control efforts, and good smoothness. A 3-D motion scene of the overall AUVs formation system under BLC protocol is depicted in Fig. \ref{fig_7_}.

\begin{figure*}[!htbp] 
	\centering
	\subfloat[]{\includegraphics[width=0.26\textwidth]{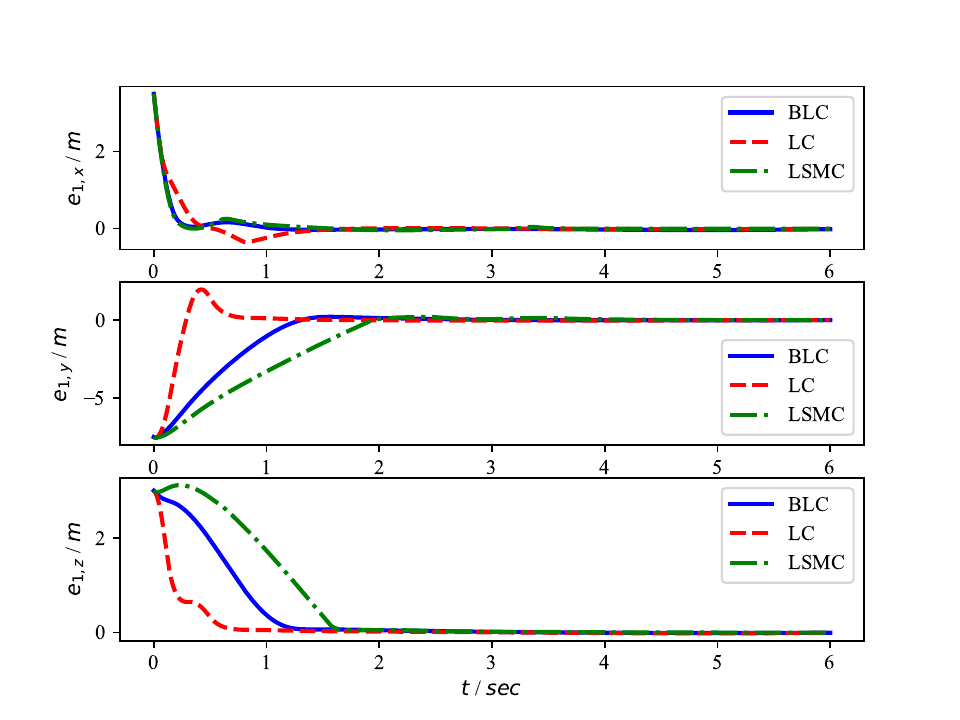}}\hspace{-0.2in} 
	\subfloat[]{\includegraphics[width=0.26\textwidth]{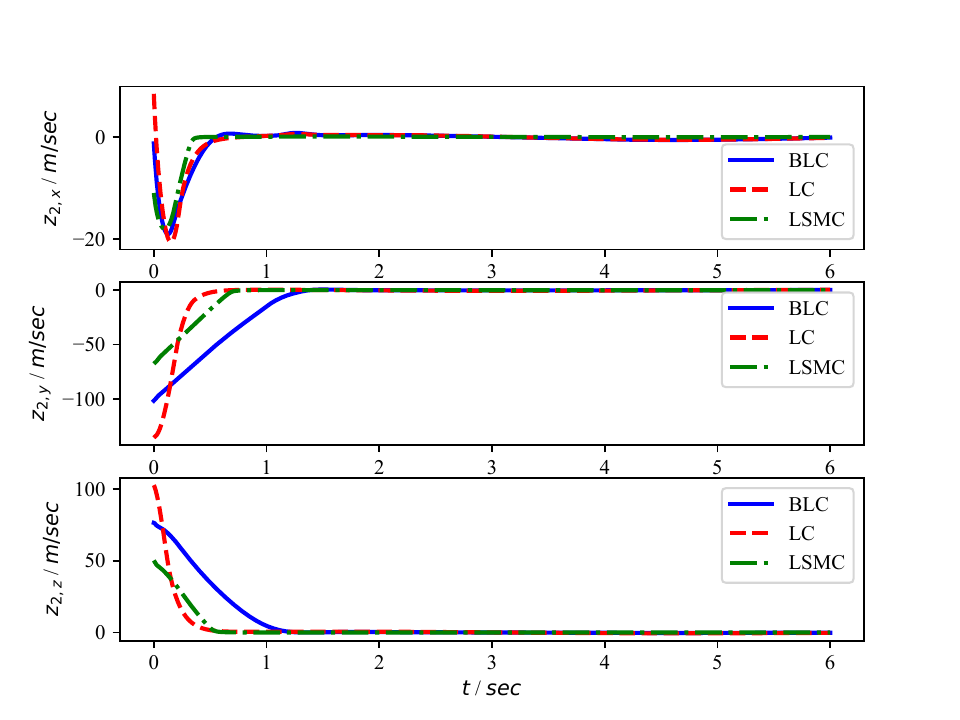}}\hspace{-0.2in} 
	\subfloat[]{\includegraphics[width=0.26\textwidth]{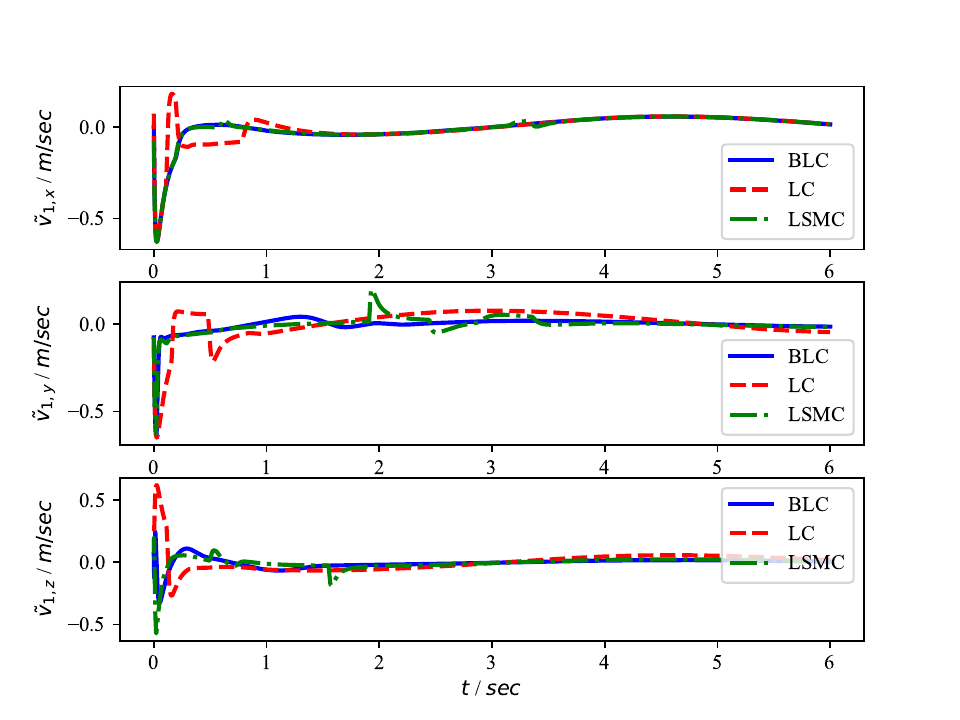}}\hspace{-0.2in} 
	\subfloat[]{\includegraphics[width=0.26\textwidth]{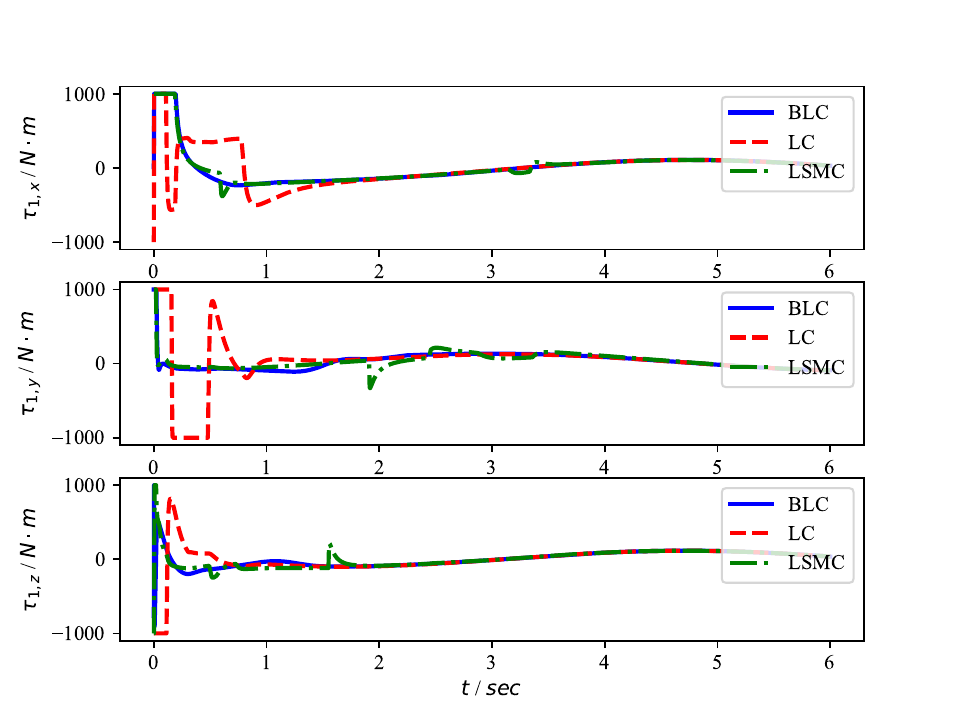}}
	\caption{The formation control performance of AUV 1 under three types of controllers. (a) The consensus formation tracking error. (b) The auxiliary variable. (c) The observation error. (d) The control signals.}
	\label{fig_8}
\end{figure*} 

Furthermore, in order to investigate the robustness performance of the proposed methodology, in the next two cases, the environmental disturbances and the noised measurements are involved in the formation system, respectively. It should be noticed that due to the fact that the proposed formation system is realized in a fully distributed manner, it will not lose the generality to show simply the performance of AUV 1, since actually all vessels behave in a very similar manner. The periodic external disturbances, induced by the ocean currents and waves, are described by the signals $d_i = \left[ {{110\sin \left( t \right),110\cos \left( t \right),110\sin \left( t \right)}}  \right.$, $\left. {{0.5\sin \left( t \right),0.5\cos \left( t \right),0.5\sin \left( t \right)}}  \right]$, $\left( {i \in \left\{ {1,2,3,4} \right\}} \right)$. It can be seen from Fig. \ref{fig_8} that all three schemes exhibit a robust behavior in terms of disturbance rejections. That is, while there exists the periodic bounded sin-type disturbance, the consensus tracking errors and auxiliary variables can still be driven into a very small neighborhood of the origin, as shown in Fig. \ref{fig_8}(a)(b). In particular, the developed learning procedure also works well when confronted with the disturbance. However, if we step further into the control behavior as illustrated in Fig. \ref{fig_8}(d), much more control efforts are needed for the LC approach to obtain this robust performance. Unsmooth control activities are observed in the LSMC scheme and further render an unsmooth learning process as seen in Fig. \ref{fig_8}(c). In contrast, the proposed BLC approach exhibits a far more consistent performance as in the unperturbed situation.
\begin{figure*}[!htbp] 
	\centering
	\subfloat[]{\includegraphics[width=0.26\textwidth]{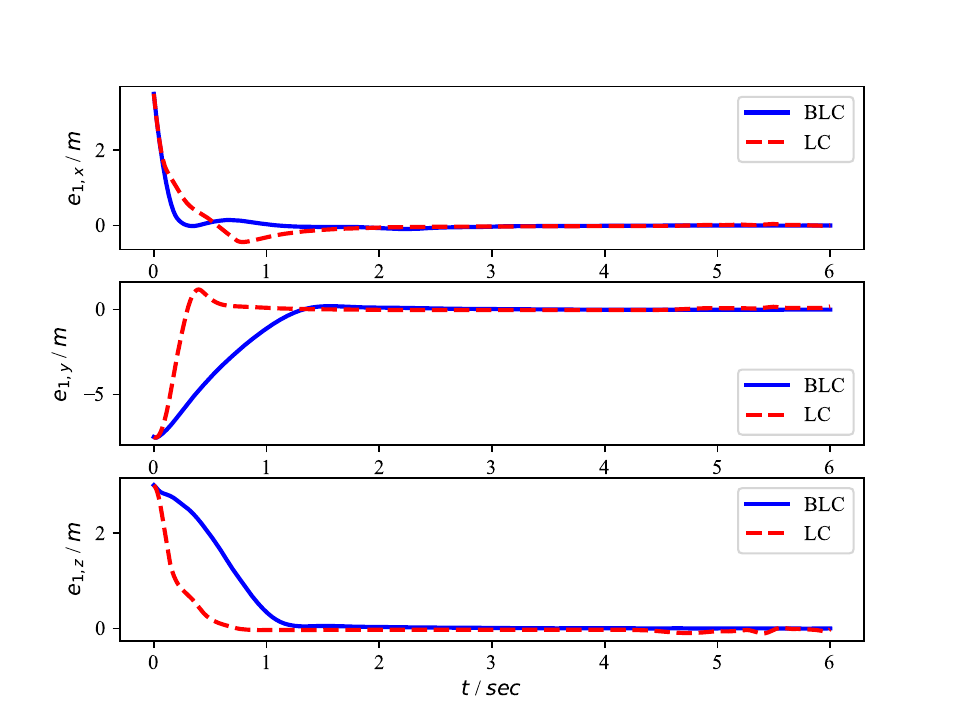}}\hspace{-0.2in} 
	\subfloat[]{\includegraphics[width=0.26\textwidth]{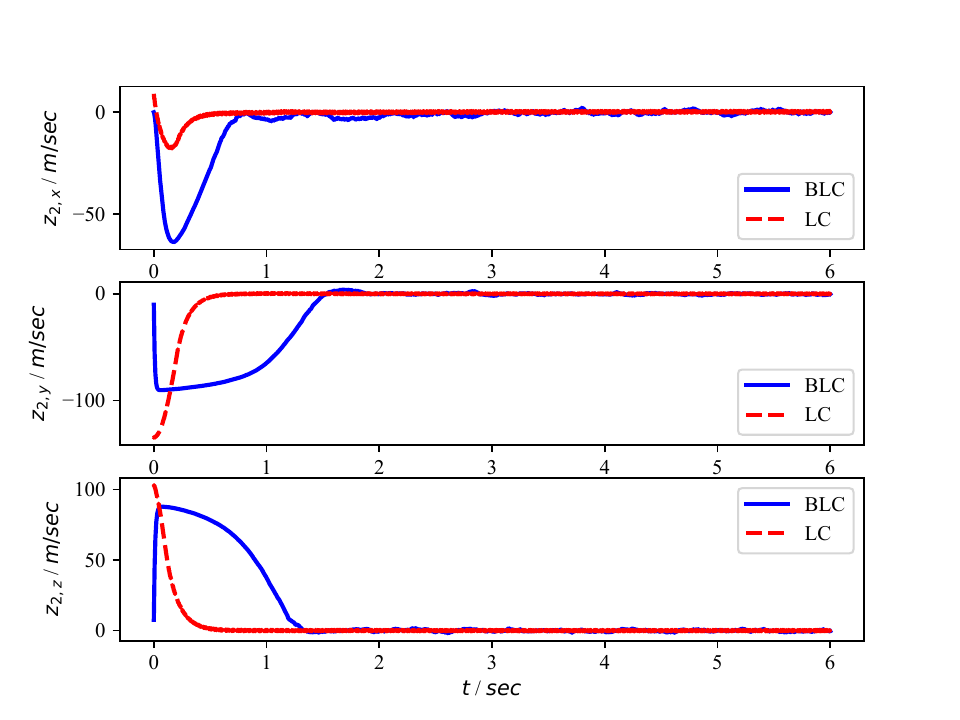}}\hspace{-0.2in} 
	\subfloat[]{\includegraphics[width=0.26\textwidth]{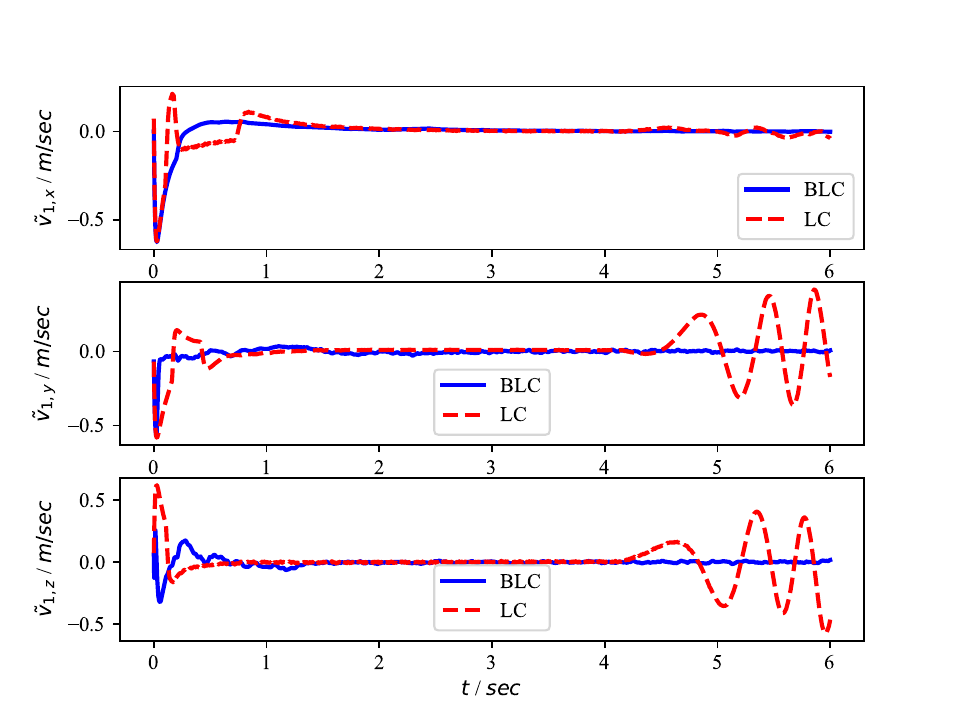}}\hspace{-0.2in} 
	\subfloat[]{\includegraphics[width=0.26\textwidth]{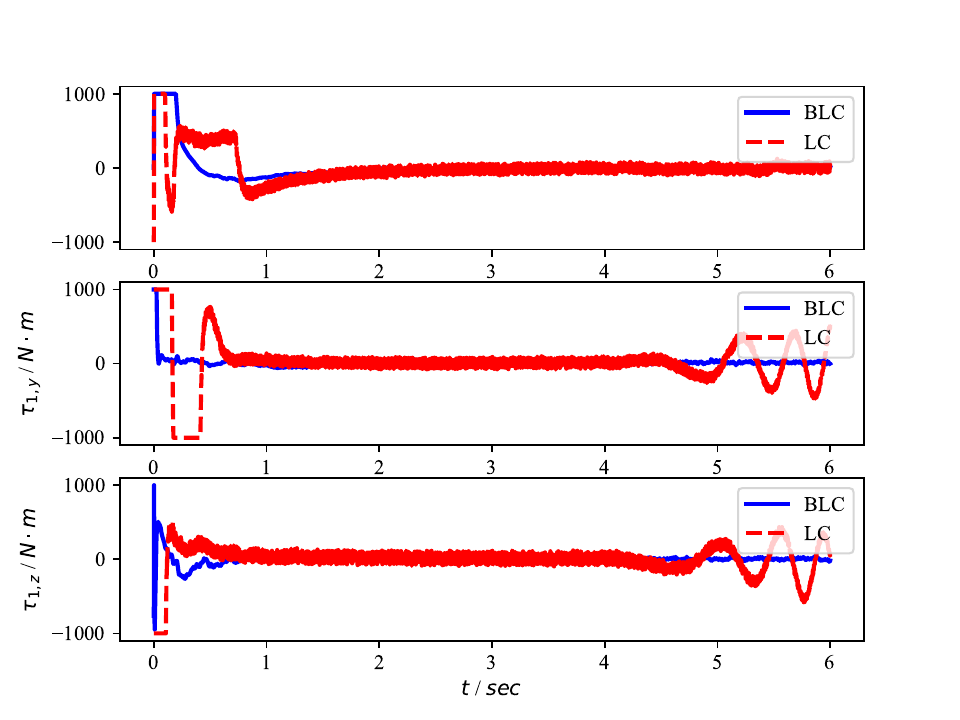}}
	\caption{The formation control performance of AUV 1 under BLC and LC controllers. (a) The consensus formation tracking error. (b) The auxiliary variable. (c) The observation error. (d) The control signals.}
	\label{fig_9}
\end{figure*}

\begin{figure*}[!htbp] 
	\centering
	\subfloat[]{\includegraphics[width=0.26\textwidth]{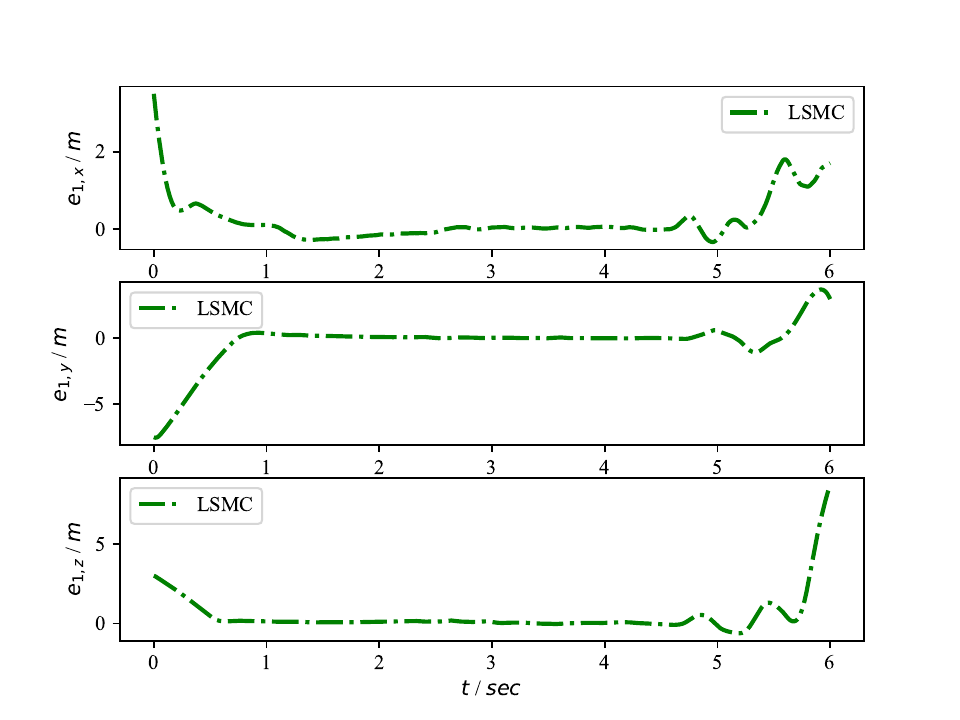}}\hspace{-0.2in} 
	\subfloat[]{\includegraphics[width=0.26\textwidth]{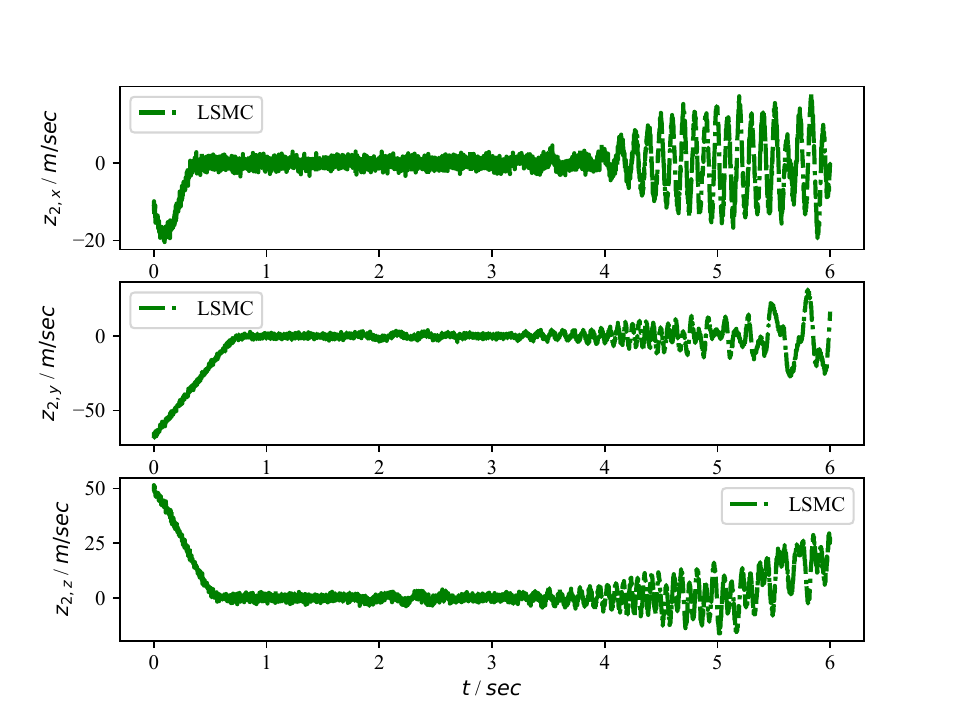}}\hspace{-0.2in} 
	\subfloat[]{\includegraphics[width=0.26\textwidth]{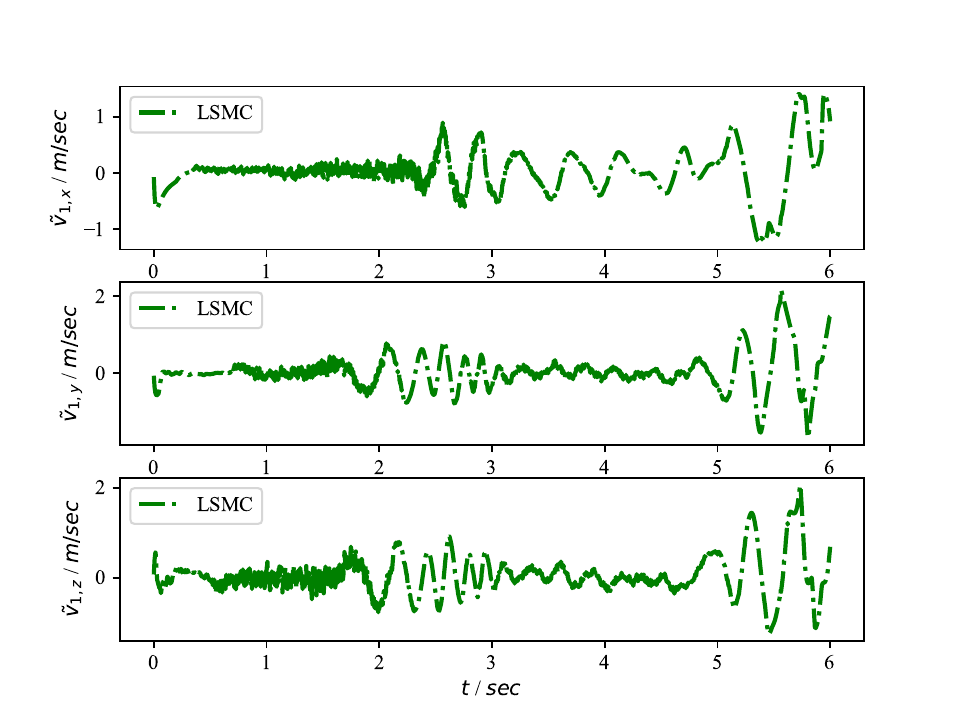}}\hspace{-0.2in} 
	\subfloat[]{\includegraphics[width=0.26\textwidth]{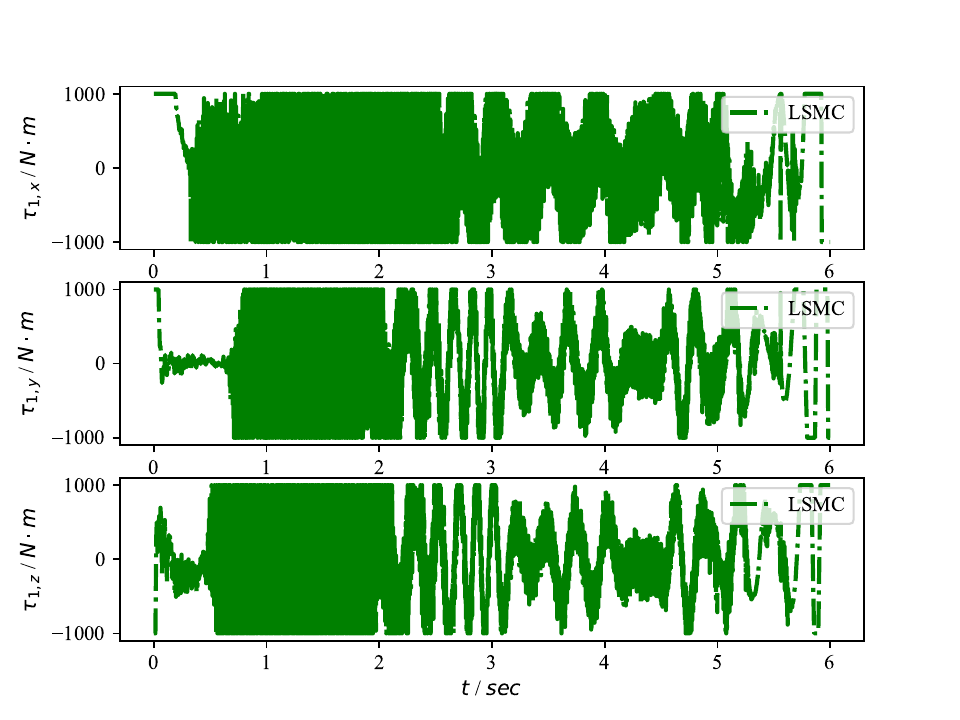}}
	\caption{The formation control performance of AUV 1 under LSMC controller. (a) The consensus formation tracking error. (b) The auxiliary variable. (c) The observation error. (d) The control signals.}
	\label{fig_10}
\end{figure*}

In the third case, the Gaussian measurement noise is injected into the control process to verify the robustness properties of the formation system in terms of noise suppression. The simulation results are shown in Figs. \ref{fig_9} and \ref{fig_10}, from which we observe that only the proposed BLC approach can achieve the consensus formation tracking objectives; that is, both LC and LSMC approaches fail to stabilize the formation system. It can be seen clearly from Fig. \ref{fig_9} that the LC approach can drive the consensus formation tracking errors into the zeros in the first four seconds, after which due to the persistent perturbation the system becomes unstable. The same results can also be given rise to in the LSMC driven formation system. Particularly, it can be shown from Fig. \ref{fig_10}(d) that the LSMC strategy is extremely sensitive to the noised measurements because of the intensive use of high frequency control activities, which results in a noised behavior in the learning process as illustrated in Fig. \ref{fig_10}(c) and ultimately ends up with an unstable system. In comparison, the proposed BLC solution exhibits a superb robustness against the Gaussian type noises; that is, the control activities are fairly smoother than the LC approach and the effects of the noises are sufficiently suppressed as shown in \ref{fig_9}(d), which also ensures a smooth learning process.

It is also observed in the simulations that the estimation gain matrices, i.e., $L_i$ and $P_i$ should be tuned first in order to ensure a smooth convergence for a successful learning process. Then, the selections of the control gains, i.e., $K_{1,i}$ and $K_{2,i}$, are dependent on the admissible control efforts as well as the desired robustness. It is clear that using the large control gains leads to better robustness properties but demands more control energy. It is worth noting that due to the integration with the shunting model, the proposed BLC approach, as shown in Fig. \ref{fig_8}, can deal with this trade-off effectively. That is, the parameter $\Lambda$ is adjustable to improve the overall system robustness while employing relatively small values of the control gains, which is beneficial for practical applications.

\section{CONCLUSION}\label{s6}
This paper is concerned with robust learning consensus formation tracking of fleets of marine vessels in 3D space where the dynamic parameters in the 6 DOF motion equations of vessels are considered to be totally unknown and subject to slow variations, and in addition, the impacts from the modeling errors, external disturbances, and measurement noises are taken into account. To this end, a novel fully distributed bio-inspired formation control protocol equipped with an online learning procedure is proposed. In more specific terms, the developed online learning procedure enables a real-time system identification so that the difficulties caused by the parameter unavailability and variations are handled effectively, and the steady formation accuracy can be thereby improved by applying an equivalence control law. Then, to obtain a robust solution against uncertainties and sensing noises while maintaining moderate control efforts, a neurodynamics model is integrated and the order of the resulting closed-loop system is thereby extended. The stability of the proposed distributed formation protocol is established to offer a theoretical guarantee for the desired robust adaptive formation performance. Furthermore, several commonly used nonlinear control schemes are compared by extensive simulation experiments, demonstrating the effectiveness and superiority of the presented methodology in terms of disturbance rejection, noise suppression, control activities, and formation accuracy. In the future, a more practical communication mechanism should be considered, for example, in the case when the out-degree information in the graph is difficult to access or subject to switching.




 
\bibliographystyle{IEEEtran}
\bibliography{reference}

\begin{IEEEbiography}[{\includegraphics[width=1in,height=1.25in,clip,keepaspectratio]{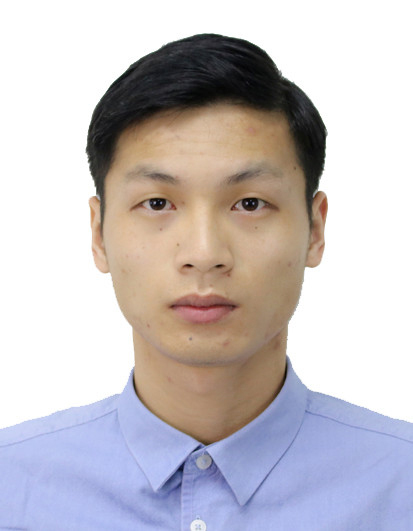}}]{Tao Yan}
(Graduate Student Member, IEEE) received the B.S. degree in automation from the
North China Institute of Aerospace Engineering, Langfang, China, in 2016, and the
M.S. degree in control science and engineering from the Zhejiang University of Technology, Hangzhou, China, in 2020. He is currently pursuing
his Ph.D. degree at the University of Guelph, ON, Canada. His research interests include the nonlinear control, machine learning, distributed control and optimization, optimal estimation, and networked underwater vehicle systems.
\end{IEEEbiography}

\vskip -2\baselineskip plus -1fil

\begin{IEEEbiography}[{\includegraphics[width=1in,height=1.25in,clip,keepaspectratio]{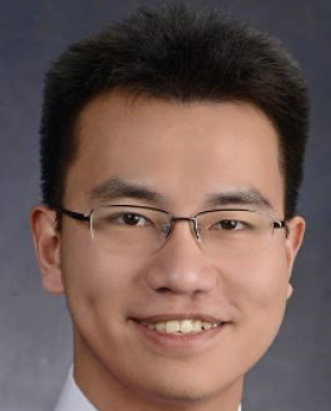}}]{Zhe Xu}
(Member, IEEE)  received B.ENG. degree in Mechanical Engineering in 2018, and M.A.Sc. and Ph.D degree in Engineering Systems and Computing in 2019 and 2023, respectively, from University of Guelph. He is currently a post-doctoral fellow with Department of Mechanical Engineering at McMaster University. His research interests include networked systems, tracking control, estimation theory, robotics, and intelligent systems.
\end{IEEEbiography}

\vskip -2\baselineskip plus -1fil

\begin{IEEEbiography}[{\includegraphics[width=1in,height=1.25in,clip,keepaspectratio]{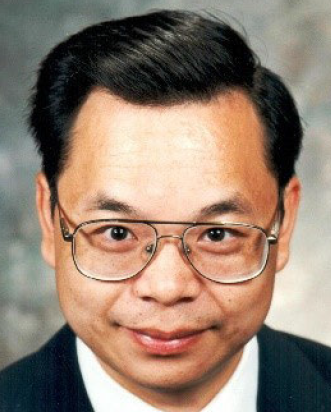}}]{Simon X. Yang}
(Senior Member, IEEE) received the B.Sc. degree in engineering physics from Beijing University, Beijing, China, in 1987, the first of two M.Sc. degrees in biophysics from the Chinese Academy of Sciences, Beijing, in 1990, the second M.Sc. degree in electrical engineering from the University of Houston, Houston, TX, in 1996, and the Ph.D. degree in electrical and computer engineering from the University of Alberta, Edmonton, AB, Canada, in 1999.  He is currently a Professor and the Head of the Advanced Robotics and Intelligent Systems (ARIS) Laboratory at the University of Guelph, Guelph, ON, Canada. His research interests include robotics, intelligent systems, control systems, sensors and multi-sensor fusion, wireless sensor networks, intelligent communication, intelligent transportation, machine learning, fuzzy systems, and computational neuroscience. 

Prof. Yang he has been very active in professional activities. He serves as the Editor-in-Chief of \textit{Intelligence \& Robotics}, and \textit{International Journal of Robotics and Automation}, and an Associate Editor of \textit{IEEE Transactions on Cybernetics}, \textit{IEEE Transactions of Artificial Intelligence}, and several other journals. He has involved in the organization of many international conferences.
\end{IEEEbiography}



\end{document}